\documentclass[10pt,twocolumn,letterpaper]{article}
\pdfoutput=1   %

\usepackage[pagenumbers]{cvpr}  %
\usepackage{cuted}       %
\usepackage{caption}
\usepackage{graphicx}
\usepackage{amsmath,amssymb}
\usepackage{booktabs}
\usepackage{multirow}
\usepackage{xcolor}
\usepackage{pifont}
\usepackage{xspace}
\usepackage{subcaption}

\usepackage{colortbl}    %

\newcommand{\cmark}{\ding{51}}
\newcommand{\xmark}{\ding{55}}
\newcommand{\gain}[1]{{\color{ForestGreen}\textbf{#1}}}   %

\newcommand{\ours}{LookZoom\xspace}          %
\newcommand{\look}{\texttt{<look>}\xspace}

\newcommand{\tbc}[1]{{\color{red}\textbf{TBC}}}

\newcommand{\B}{B}                            %
\newcommand{\gl}{g_{\text{look}}}            %
\newcommand{\gans}{g_{\text{ans}}}           %

\usepackage{amsthm}             %
\newtheorem{proposition}{Proposition}

\definecolor{cvprblue}{rgb}{0.21,0.49,0.74}
\usepackage[breaklinks,colorlinks,allcolors=cvprblue]{hyperref}   %
\hypersetup{pdftitle={Zoom Where the Answer Looks: Cross-Resolution Self-Distillation for MLLMs under Visual-Token Budgets},
            pdfauthor={Zhensheng Hu, Guanglong Xu, Guanqiang Hu}}

\newcommand{\papertitle}{Zoom Where the Answer Looks:\\Cross-Resolution Self-Distillation for MLLMs under Visual-Token Budgets}
\title{\papertitle}

\author{
Zhensheng Hu$^{1,2}$\qquad Guanglong Xu$^{3}$\qquad Guanqiang Hu$^{4,*}$\\[0.3em]
{\normalsize $^{1}$Integrated Hospital of Traditional Chinese Medicine, Southern Medical University, Guangzhou, China}\\
{\normalsize $^{2}$Institute of Medical Engineering Interdisciplinary Research, Xinjiang Medical University, Urumqi, China}\\
{\normalsize $^{3}$School of Mathematics and Statistics, Zhaoqing University, Guangdong 526061, China}\\
{\normalsize $^{4}$Guangdong Kamfu Technology Co., Ltd., Foshan, China}\\
{\tt\small huzhensheng@smu.edu.cn}\quad{\tt\small xugl2018@163.com}\quad{\tt\small guanqiang.hu.cn@gmail.com}\quad{\small $^{*}$Corresponding author}
}

\begin{document}
\maketitle
\widowpenalty=10000 \clubpenalty=10000 \displaywidowpenalty=10000   %
\raggedbottom   %
\begin{abstract}
Under the small per-view visual-token budgets of efficient deployment, the fine text of documents, charts and infographics becomes illegible to multimodal large language models (MLLMs).
Zooming in recovers it, but the region must be chosen before the model answers.
Existing methods decode an answer or a box first, at the cost of generation, or read the attention available before answering, which localizes the evidence only in top layers whose depth varies with model size.
Learned region proposers instead distill the attention of the model's answers, but on the same budgeted view.
Yet this attention is precise only on a view the model can read: on Qwen3-VL-4B, a window around its peak covers $83\%$ of the evidence on a high-resolution view but $54\%$ on the budgeted view.
Since answers and high-resolution views are available during training, we introduce \ours, a zooming method that trains a dedicated look token by cross-resolution self-distillation of the answer's attention: with the MLLM frozen, the token learns to attend on the budgeted view, before any answer, where the model's own correct answers attend on a high-resolution view.
The region around its attention peak is the window to zoom into, and reinforcement learning then trains the MLLM to answer from both views.
With 98 visual tokens per view on five text-rich benchmarks, \ours surpasses attention-guided zooming by $20$--$24$ points and grounding-based zooming and learned crop routing by $12$--$15$ points on Qwen3-VL at three scales, and it outperforms learned region proposers at every budget from 98 to 392 tokens.
On Qwen3-VL, the look token alone, with the MLLM frozen, already surpasses every baseline that leaves the language model unchanged, and the recipe carries over to InternVL3.5 and Qwen2.5-VL.
\end{abstract}

\section{Introduction}
\label{sec:intro}
\begin{figure}[t]
  \centering
  \includegraphics[width=\linewidth]{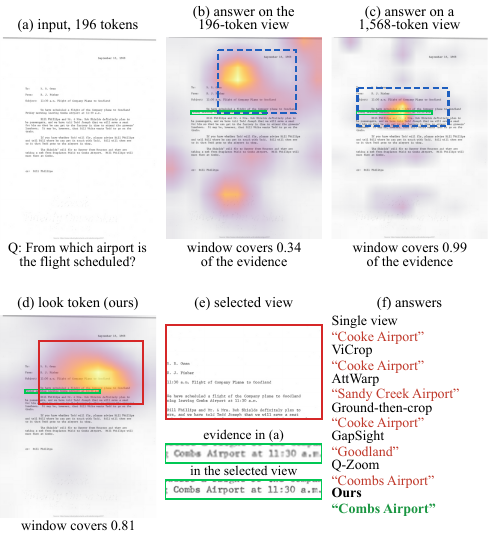}
  \caption{\textbf{Where the answer looks} (Qwen3-VL-4B, V-CoT-Doc, 196 visual tokens per view). (b)~On the budgeted view, the answer's attention falls next to the evidence (green box), and its support window (dashed) covers a third of it. (c)~On a 1{,}568-token view of the same image, pooled to the 196-token grid, it covers the evidence; this is the teacher of our \look token. (d)~The \look token reproduces it on the budgeted view in a single prefill, before any answer, and its support window (red) is the region to zoom into. (e)~The selected view makes the line legible, and (f)~only the answer from it is correct.}
  \label{fig:teaser}
\end{figure}

Reading a receipt, a chart or an infographic requires seeing small text.
Multimodal large language models (MLLMs) see an image through visual tokens whose number grows with its resolution~\cite{bai2025qwen25vl,bai2025qwen3vl,liu2024llavanext}, and long documents, multi-image prompts, video and on-device inference all cap the number of tokens each image may use.
Under such a budget the text that answers the question often becomes illegible: with 98 visual tokens per view, Qwen3-VL-8B averages $47.9$ on five text-rich benchmarks, against $81.8$ with 392 tokens.

A natural remedy is to \emph{look again}: re-encode the relevant region at a higher effective resolution within the same budget, which requires choosing the region before the model answers.
Decoding-based methods name the region with box coordinates~\cite{shao2024visualcot,yu2025zoomrefine,zheng2026deepeyes,zhao2026lfpc,lin2026adaptvision} or generate an answer and follow its attention~\cite{dalal2026attwarp}, at the cost of a decoding pass on the illegible view.
Training-free attention methods read the attention available before answering~\cite{zhang2025vicrop,liu2026hide}, which localizes the evidence only in top layers whose depth varies with model size.
Learned region proposers~\cite{shi2026sdrpn,shi2026qzoom} distill the attention of the model's answers into a predictor that runs before answering, but take this attention on the same budgeted view.

The attention of the answer localizes the evidence already in the middle layers, but precisely only on a view the model can read (Fig.~\ref{fig:teaser}).
On Qwen3-VL-4B, a window around its peak covers $0.83$ of the evidence when the answer is produced on a high-resolution view of the image, but $0.54$ when it is produced on the budgeted view, where the model attends to the right neighborhood but not to the right cells (Sec.~\ref{sec:observations}).
A pointer distilled on the budgeted view inherits this imprecision, whereas the answer and a high-resolution view are both available during training.

We therefore introduce \ours, a zooming method that learns to look, on the budgeted view and before any answer, where the model's answer looks on a high-resolution view.
\ours appends a dedicated \look token to the prompt, with its own position-specific parameters and removed from the answering pass, so learning to point does not change how the model reads, unlike supervising the shared attention, which costs reading accuracy.
The token is trained by \emph{cross-resolution self-distillation} of the answer's attention: with the model frozen, it learns to attend on the budgeted view where the model's own correct answers attend on a high-resolution view of the same image, without box annotations or external networks.
The support of its attention, the connected cells around the peak that hold at least twice the uniform mass, is the window to zoom into, and reinforcement learning then adapts the \emph{reader}, the MLLM that answers from the budgeted and the zoomed views.

\ours achieves the best average accuracy at every budget from 98 to 392 tokens on three Qwen3-VL scales and five text-rich benchmarks (Tab.~\ref{tab:main}).
At 98 tokens per view, it leads attention-guided zooming by $20.0$--$23.8$ points and grounding-based zooming and GapSight~\cite{zhu2026gapsight} by $12.0$--$14.9$, and it leads the learned region proposers at every budget, by up to $6.3$ points over Q-Zoom, which also fine-tunes the language model, on Qwen3-VL-4B and $13.5$ over SD-RPN on Qwen2.5-VL-7B.
The \look token carries most of this gain: with the reader frozen, \ours already surpasses every baseline that leaves the language model unchanged, and a high-resolution teacher adds up to $6.5$ points over a teacher on the budgeted view (Sec.~\ref{sec:ablation}).
\ours needs no generation before the answer, is faster than attention-guided and grounding-based zooming, and carries over to InternVL3.5 and Qwen2.5-VL.

\noindent Our contributions are:
\begin{itemize}
  \item An analysis of where MLLMs localize the evidence under visual-token budgets, before and while answering and across resolutions.
  \item \ours, a \look token trained on the frozen model by cross-resolution self-distillation of the answer's attention, which points in one prefill of the budgeted view and sets where and how far to zoom.
  \item The best average accuracy at every budget on three Qwen3-VL scales, InternVL3.5 and Qwen2.5-VL, ahead of learned region proposers, with ablations that trace the gain to the resolution of the teacher.
\end{itemize}

\section{Related work}
\label{sec:related}

\noindent\textbf{Visual tokens and resolution.}
Text-rich understanding depends on the number of visual tokens, which grows with resolution~\cite{bai2025qwen25vl,bai2025qwen3vl,liu2024llavanext,zhang2025vicrop}.
Token pruning and merging~\cite{chen2024fastv,yang2025visionzip,zhang2025sparsevlm,shang2025prumerge} encode the full image and are complementary to our per-view budgets; VisionThink~\cite{yang2025visionthink} and AdaptVision~\cite{lin2026adaptvision} learn with RL when to request the full image or a region from a low-resolution input.
We study where and how to look again within a fixed per-view budget.

\noindent\textbf{Looking again.}
Guided visual search~\cite{wu2024vstar} and chain-of-thought with evidence boxes~\cite{shao2024visualcot} showed that re-encoding a relevant region helps MLLMs see small details.
Many methods choose the region by decoding text first: Zoom-Refine~\cite{yu2025zoomrefine} predicts a box with a preliminary answer, ZoomEye~\cite{shen2025zoomeye} searches a tree of sub-patches, and RL-trained agents emit box coordinates during multi-turn reasoning~\cite{zheng2026deepeyes,zhang2025chainoffocus,zhao2026lfpc}.
Training-free methods derive the region from internal signals: the attention of the answer-start position (ViCrop~\cite{zhang2025vicrop}), of individual question tokens (HiDe~\cite{liu2026hide}) or of a generated answer (AttWarp~\cite{dalal2026attwarp}), or a relevance map from the key-value cache (FOCUS~\cite{zhong2025focus}); LookWise~\cite{shen2026lookwise} and a gated variant of AttWarp also decide whether to zoom.
GapSight~\cite{zhu2026gapsight} learns whether and where to crop from the loss gaps of candidate crops.
\ours instead points with a dedicated token in the prefill of the question, without decoding text, and zooms into the support of its pointer, without a set of scales or a learned zoom decision.

\noindent\textbf{Learning where to look.}
The attention of MLLMs localizes visual evidence~\cite{zhang2025vicrop,kang2025fewheads,liu2025seeingbutnot}, and the attention of the response is the more informative signal for distillation~\cite{jang2026trag}.
SD-RPN~\cite{shi2026sdrpn} distills the attention of the model's own responses, taken on the same view that the predictor sees, into an added region-proposal branch; Q-Zoom~\cite{shi2026qzoom} builds on it with a gate that skips high-resolution processing and fine-tunes the language model on hard samples, and region-level policy optimization~\cite{shi2026visionrl2} refines such a predictor with RL.
The answer's attention, however, is precise only at a resolution where the evidence is legible (Sec.~\ref{sec:observations}), so \ours distills it across resolutions, from the model's own correct answers on a high-resolution view (\cf hindsight distillation~\cite{zhao2025hind}).
Matching a high-resolution teacher in a low-resolution student is an established recipe for low-resolution recognition~\cite{shin2022teachingwhere,delamor2023interresolution}, and Vision-OPD~\cite{yuan2026visionopd}, RP-OPSD~\cite{zhu2026rpopsd} and region-to-image distillation~\cite{wei2026zoomingwithout} distill the outputs of a teacher that sees a crop or the full-resolution image, so that the student answers better without zooming; we keep the zoom and transfer only where to look.
Other work shapes the shared attention during training against annotated boxes~\cite{gong2026saliencyr1,ou2026sayo,hu2026ease,selvaraju2019hint,zhou2025guiaima}, a box-trained detector~\cite{wu2026confothinking} or an external grounder~\cite{berger2026selfsaliency}; \ours leaves it untouched, since in our setting supervising it costs reading accuracy (Sec.~\ref{sec:look}).

\section{Preliminary analysis}
\label{sec:observations}
We first examine where an MLLM localizes the evidence when its visual tokens are limited, using Qwen3-VL~\cite{bai2025qwen3vl} without any training.
To measure where the model looks, we use V-CoT-Doc, a held-out document split of Visual CoT~\cite{shao2024visualcot} whose 300 questions come with evidence boxes, and define the \emph{evidence mass} as the fraction of the attention to the image that falls inside the evidence box.
Unless stated otherwise, attention maps are averaged over heads and over the band of middle layers from which \ours takes its pointer, and smoothed identically for all variants; a uniform map has an evidence mass of $0.040$ at 196 tokens.

\begin{figure}[t]
  \centering
  \includegraphics[width=\linewidth]{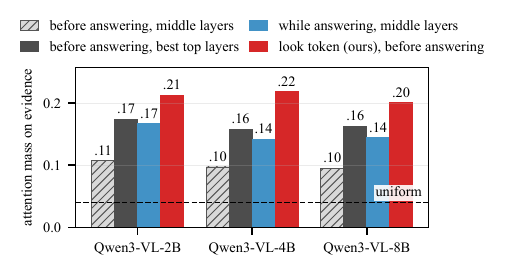}
  \caption{\textbf{Where the model looks before and while answering} (V-CoT-Doc, 196 tokens). In the middle layers, the attention before answering places little mass on the evidence, while the attention of the answer tokens already localizes it. Before answering, comparable localization appears only in the top layers (best band of four layers per model). Our \look token, taken from the middle layers, concentrates the most mass on the evidence \emph{before} any answer is produced.}
  \label{fig:attn}
\end{figure}

\noindent\textbf{The answer localizes the evidence from the middle layers on, the pre-answer attention only in the top layers.}
A zooming decision has to be made before the answer exists, \eg from the attention of the answer-start position~\cite{zhang2025vicrop}.
In the middle layers, this \emph{pre-answer attention} places only $0.107$ / $0.097$ / $0.095$ of its mass on the evidence (2B / 4B / 8B at 196 tokens, Fig.~\ref{fig:attn}), whereas the attention of the model's own answer tokens places $1.5$--$1.6{\times}$ more ($0.167$ / $0.142$ / $0.144$), nearly as much as when the ground-truth answer is teacher-forced ($0.169$ / $0.144$ / $0.145$).
Before answering, comparable localization appears only in the top layers, at a depth that differs across scales (layers 17--20 of 28 for 2B, 21--24 and 28--31 of 36 for 4B and 8B; per-layer curves in supplementary Sec.~B), in line with reports that deeper layers attend to the evidence~\cite{liu2025seeingbutnot}.
\ours therefore takes the answer's attention as the teacher of a token whose attention is taken from a fixed band of middle layers before any answer exists (Sec.~\ref{sec:distill}).

\noindent\textbf{The answer localizes the evidence precisely only on a view it can read.}
The first analysis measured how much attention the answer places on the evidence; we now ask how precisely, as a function of the resolution at which the answer is produced.
On 300 further held-out questions of the same Visual CoT split, used neither in V-CoT-Doc nor in training, we teacher-force the ground-truth answer on views at 98 to 1{,}568 tokens, pool the answer-time attention onto the grid of the 98-token view, and take its support window $\texttt{supp}_2$ (Sec.~\ref{sec:window}).
The fraction of attention on the evidence barely changes with resolution ($0.28$ to $0.27$ on 4B, unsmoothed), but its placement sharpens: the support window covers $0.54$ of the evidence when the answer is produced on the 98-token view, $0.70$ on the 392-token view and $0.83$ on the 1{,}568-token view, with the window growing only from $6$ to $9\%$ of the image (2B: $0.51\to0.82$; 8B: $0.47\to0.79$; supplementary Sec.~B).
On the budgeted view, the model attends to the right neighborhood but not to the right cells, since the text there is illegible; giving the evidence more pixels likewise sharpens the attention in AttWarp~\cite{dalal2026attwarp}.
A pointer distilled from the answer on that view inherits this blurred target, so \ours distills the answer's localization on a high-resolution view into the pointer on the budgeted view (Sec.~\ref{sec:distill}).

\section{Method}
\label{sec:method}

\begin{figure*}[t]
  \centering
  \includegraphics[width=\textwidth]{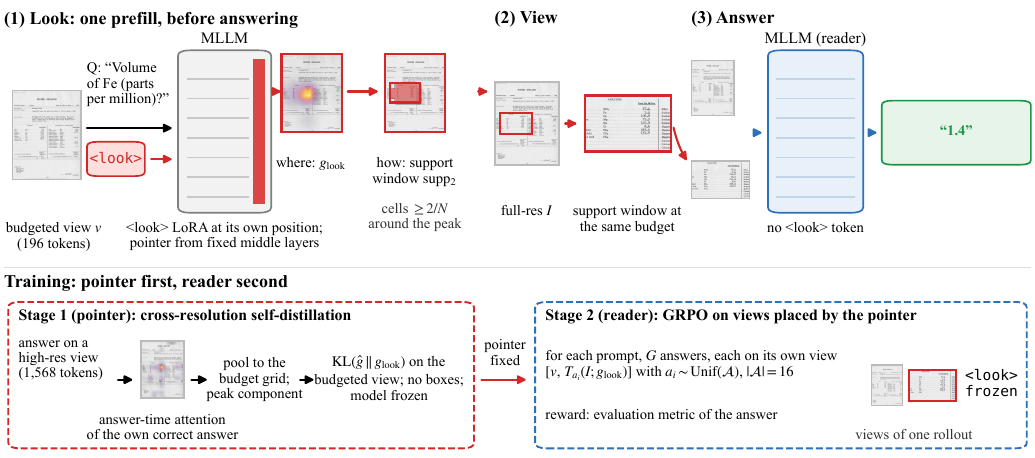}
  \caption{\textbf{Overview of \ours.} (1)~A dedicated \look token is processed in a single prefill on the budgeted view $v$; its attention to the image tokens, the pointer $\gl$, tells \emph{where} to look, and the pointer's support tells \emph{how}. (2)~The support window is cut from the full-resolution image and re-encoded within the same per-view budget. (3)~The reader answers from $v$ and the window; the \look token never enters this pass. Stage~1 distills the attention of the model's own correct answers on a high-resolution view into the \look token, with the model frozen; stage~2 adapts the reader with GRPO, with the \look token fixed.}
  \label{fig:method}
\end{figure*}

\ours answers from the budgeted view and one zoomed window (Fig.~\ref{fig:method}).
A dedicated \look token points to the evidence before the answer (Sec.~\ref{sec:look}), trained to attend where the answer attends on a high-resolution view (Sec.~\ref{sec:distill}); its support gives the window (Sec.~\ref{sec:window}), and reinforcement learning then adapts the reader to the window (Sec.~\ref{sec:reader}).

\subsection{Problem setup}
\label{sec:setup}
Given an image $I$, a question $q$ and a visual-token budget $\B$ per view, the MLLM sees the \emph{budgeted view} $v=\mathrm{resize}(I,\B)$, encoded with at most $\B$ visual tokens.
Following prior zooming work~\cite{zhang2025vicrop}, a second view is formed from the full-resolution image, again within $\B$ tokens, and the answer is generated from $[\,v,\ T_a(I;g)\,]$.
Here the \emph{pointer} $g\in\Delta^{h\times w}$ is a distribution over the $N=hw$ cells of the image-token grid of $v$, and the \emph{view transform} $T_a$ maps $I$ and $g$ to the second image, \eg a crop of fixed size with the largest pointer mass; $a=\texttt{none}$ answers from $v$ alone.
Two decisions must thus be made for every sample before the answer is known: \emph{where} to look ($g$) and \emph{how} to look ($a$).

\subsection{A dedicated look token}
\label{sec:look}
We append a learnable token \look to the prompt after the assistant prefix of the chat template and run a single prefill on $[v,q,\look]$.
Its input embedding is that of an unused special token plus a learnable offset; the sum is initialized to the embedding of the word ``look''.
A LoRA~\cite{hu2022lora} on every attention and MLP projection of decoder layers $0$ to $\ell_{\max}$, the last layer of the pointer band, acts only at the \look position.
The pointer is the attention of \look to the image tokens, averaged over the $K$ heads of a band of middle layers $\Lambda$ and normalized over the image,
\begin{equation}
\gl = \mathcal{N}\Big(\sum_{\ell\in\Lambda}\sum_{k=1}^{K} a^{(\ell,k)}\Big),
\label{eq:pointer}
\end{equation}
where $a^{(\ell,k)}\in\mathbb{R}_{\ge 0}^{h\times w}$ holds the attention weights of head $k$ in layer $\ell$ from \look to the image tokens and $\mathcal{N}$ normalizes a map to unit sum.

\noindent\textbf{Isolating pointing from reading.}
The \look token is removed from the answering pass, so the reader never attends to it; its LoRA acts at no other position; and the token and the reader are trained in separate stages, each with the other frozen.
Pointing can therefore be trained as strongly as needed without changing how the model reads.
On 2B with RL fine-tuning of the reader, supervising the shared pre-answer attention with the answer-time attention instead sharpens pointing about as much as a dedicated token (evidence mass $0.217$ \vs $0.221$ at 196 tokens), but changes single-view accuracy by $-0.7$ instead of $+0.1$ points, relative to the same fine-tuning without pointer losses (mean over 11 benchmark--budget settings; supplementary Sec.~C).

\subsection{Where to look: cross-resolution self-distillation}
\label{sec:distill}
Sec.~\ref{sec:observations} shows that the answer tokens localize the evidence from the middle layers on, and precisely only on a view they can read.
We therefore distill into the \look token the answer's attention \emph{on a high-resolution view}, rather than on the budgeted view itself as in prior self-distillation~\cite{shi2026sdrpn,shi2026qzoom}, and ask the token to reproduce it on the budgeted view that the model actually sees.

For a training sample $(I,q,y^\star)$ we build a teacher view $v^{\uparrow}=\mathrm{resize}(I,\B^{\uparrow})$ with a large budget $\B^{\uparrow}=1{,}568$ tokens.
The model answers greedily on $v^{\uparrow}$; its own answer $y$ is kept if it is correct ($r(y)\ge0.5$ under the task metric $r$), and the ground truth $y^\star$ is used otherwise.
Teacher-forcing $y$ on $v^{\uparrow}$ gives the answer-time attention $\gans(y;v^{\uparrow})$, the attention from the answer tokens to the image, aggregated over tokens, heads and $\Lambda$ as in Eq.~\eqref{eq:pointer}, on the $h^{\uparrow}\times w^{\uparrow}$ grid of $v^{\uparrow}$.
We pool it by area onto the $h\times w$ grid of $v$ and keep only its \emph{support}: for a map $g$ with peak cell $i^\star=\arg\max_i g_i$,
\begin{equation}
C_\tau(g)=\mathrm{cc}_{i^\star}\big\{\,i:\ g_i\ge\tau/N\,\big\}
\label{eq:support}
\end{equation}
is the connected component, containing the peak, of the cells that hold at least $\tau$ times the uniform mass.
The target and the loss of the \look token are
\begin{equation}
\begin{aligned}
\bar g&=\mathrm{pool}_{h\times w}\,\gans(y;v^{\uparrow}),\quad \hat g=\mathcal{N}\big(\bar g\odot\mathbf{1}_{C_\tau(\bar g)}\big),\\
\mathcal{L}_{\text{look}}&=\mathrm{KL}\big(\hat g\,\|\,\gl(v,q)\big),
\end{aligned}
\label{eq:target}
\end{equation}
where $\tau=2$, the four corner cells of $\bar g$, which absorb attention sinks, are zeroed, and the budget $\B$ of the student view $v$ is drawn per sample from the test-time budgets.
Teacher and student both run on the frozen base model, and only the embedding offset and the LoRA of \look receive gradients: this stage needs neither a trained reader nor a reward.
Stage~1 also works without answer labels: using every own answer, right or wrong, as the teacher changes accuracy by at most $0.4$ points (Sec.~\ref{sec:ablation}).

\subsection{How to look: the support window}
\label{sec:window}
How much to magnify differs across samples: a zoomed view adds detail only where the information gap $\rho$, the ratio of an image's native to its budget pixels (\cf~\cite{zhao2026lfpc}), is large, and no fixed view is best for every sample (Sec.~\ref{sec:analysis}).
A window that adapts to the pointer depends on its shape: the quantile window $\texttt{adapt}_{.8}$, spanning the central $80\%$ of the pointer mass along each axis, barely magnifies when the pointer is spread out (supplementary Sec.~C).
Since the \look token is trained to reproduce the support of the answer's attention (Eq.~\eqref{eq:target}), we take the window from the support of the pointer: the support window
\begin{equation}
\texttt{supp}_\tau=\mathrm{bbox}\big(C_\tau(G_\sigma*\gl)\big)
\label{eq:window}
\end{equation}
is the bounding box of the support of the smoothed pointer (Gaussian $G_\sigma$, $\sigma=1$ cell), and \ours zooms into $\texttt{supp}_2$.
The window has no preset size or aspect ratio; it follows the extent of the pointer and shrinks to a few cells when the pointer is sharp.
The budgeted view stays in the answering pass as the first view, so even a tight window keeps its context.
The rule needs no learned zoom decision and applies unchanged to every sample and budget.

\subsection{Training and inference}
\label{sec:reader}
Training has two stages, pointer first and reader second, each training one component with the other fixed.
(1)~Stage~1 trains the \look token by cross-resolution self-distillation with the model frozen (Sec.~\ref{sec:distill}); with the base model as the reader, it yields a complete zooming method (Sec.~\ref{sec:main_results}).
(2)~Stage~2 freezes the \look token and adapts the reader, the MLLM with a standard LoRA, with GRPO~\cite{shao2024deepseekmath}: for each prompt we sample $G$ answers, each on its own view $[v, T_{a_i}(I;\gl)]$, and the reward of an answer is its score under the evaluation metric of its source dataset.
The view $a_i$ is drawn uniformly from a family $\mathcal{A}$ of sixteen views placed by the pointer (supplementary Sec.~A), from no second view, fixed-size crops and a foveated warp to $\texttt{adapt}_{.8}$ and support windows.
The reader thus learns to read the windows it receives at test time, next to wider views and the budgeted view alone.

\noindent\textbf{Inference.}
One prefill on $[v,q,\look]$ yields $\gl$; its support window $\texttt{supp}_2$ is cut from the full-resolution image and re-encoded within the budget; and one answering pass on $[v, T_{\texttt{supp}_2}(I;\gl), q]$ produces the answer.
\ours thus needs two passes and no generation before the answer.

\section{Experiments}
\label{sec:experiments}

\subsection{Setup}
\label{sec:exp_setup}
\begin{table*}[t]
\centering
\footnotesize
\setlength{\tabcolsep}{1.9pt}
\caption{\textbf{Main results under per-view visual-token budgets.} Accuracy (\%, official metrics) with at most $\B$ visual tokens per view; the column headers give $\B$. Per-benchmark results use $\B=196$ visual tokens per view; Avg5 averages the five benchmarks and Avg4 the four public ones (without V-CoT-Doc). \emph{Views}: views in the answering pass. The training-free baselines and GapSight$^\dagger$ (our re-implementation) run on the base model; Q-Zoom uses its released weights, which exist only for the 4B model. Ground-then-crop adds a crop of the evidence box named by the model. Ours, frozen reader trains only the \look token and answers with the base model (stage~1). Bold: best; underlined: second best (per model). $\Delta$: Ours minus the best baseline. $^\ast$Q-Zoom's released training data contain the V-CoT-Doc questions we evaluate on.}
\label{tab:main}
\begin{tabular}{@{}c@{\hspace{2pt}}l cc ccccc cc cc cc@{}}
\toprule
& & & & \multicolumn{5}{c}{196} & \multicolumn{2}{c}{98} & \multicolumn{2}{c}{196} & \multicolumn{2}{c}{392} \\ \cmidrule(lr){5-9} \cmidrule(lr){10-11} \cmidrule(lr){12-13} \cmidrule(lr){14-15}
& Method & Training-free & Views & V-CoT-Doc & DocVQA & InfoVQA & ChartQA & TextVQA & Avg5 & Avg4 & Avg5 & Avg4 & Avg5 & Avg4 \\
\midrule
\multirow{8}{*}{\rotatebox[origin=c]{90}{Qwen3-VL-2B}} & Single view & -- & 1 & 52.0 & 66.7 & 37.9 & \underline{77.8} & 68.3 & 39.5 & 44.9 & 60.5 & 62.7 & 74.2 & 73.3 \\
 & ViCrop~\cite{zhang2025vicrop} & \cmark & 2 & 52.3 & 66.2 & 42.2 & 71.8 & 69.4 & 41.3 & 45.9 & 60.4 & 62.4 & 72.6 & 71.8 \\
 & AttWarp~\cite{dalal2026attwarp} & \cmark & 1 & 53.7 & 65.8 & 42.9 & 76.8 & 72.6 & 39.6 & 45.3 & 62.3 & 64.5 & 76.2 & 74.8 \\
 & Ground-then-crop & \cmark & 2 & 68.7 & 72.6 & 46.5 & 70.8 & 71.8 & 50.2 & 51.5 & 66.1 & 65.4 & 76.2 & 73.9 \\
 & GapSight$^\dagger$~\cite{zhu2026gapsight} & \xmark & 2 & 63.0 & 73.1 & 46.3 & 75.8 & 67.7 & 48.4 & 48.0 & 65.2 & 65.7 & 75.2 & 74.0 \\
\cmidrule(l){2-15}
 & \textbf{Ours}, frozen reader & \xmark & 2 & \underline{82.0} & \underline{80.8} & \underline{54.3} & 69.4 & \underline{74.4} & \underline{57.8} & \underline{57.4} & \underline{72.2} & \underline{69.7} & \underline{77.8} & \underline{75.7} \\
 & \textbf{Ours} & \xmark & 2 & \textbf{83.7} & \textbf{83.0} & \textbf{59.6} & \textbf{79.6} & \textbf{75.9} & \textbf{65.1} & \textbf{65.1} & \textbf{76.3} & \textbf{74.5} & \textbf{81.5} & \textbf{79.8} \\
\rowcolor{gray!8} & $\Delta$ over the best baseline & & & \gain{+15.0} & \gain{+9.9} & \gain{+13.1} & \gain{+1.8} & \gain{+3.3} & \gain{+14.9} & \gain{+13.6} & \gain{+10.2} & \gain{+8.8} & \gain{+5.3} & \gain{+5.0} \\
\midrule
\multirow{9}{*}{\rotatebox[origin=c]{90}{Qwen3-VL-4B}} & Single view & -- & 1 & 55.7 & 70.2 & 45.6 & 82.4 & 71.1 & 44.7 & 50.9 & 65.0 & 67.3 & 78.8 & 77.9 \\
 & ViCrop~\cite{zhang2025vicrop} & \cmark & 2 & 67.3 & 71.2 & 50.9 & 78.6 & 73.2 & 50.8 & 55.7 & 68.2 & 68.5 & 79.2 & 77.9 \\
 & AttWarp~\cite{dalal2026attwarp} & \cmark & 1 & 59.7 & 71.7 & 49.1 & 79.0 & 75.6 & 45.3 & 51.1 & 67.0 & 68.9 & 81.6 & 80.0 \\
 & Ground-then-crop & \cmark & 2 & 65.7 & 78.4 & 56.8 & 79.8 & 73.9 & 57.5 & 60.8 & 70.9 & 72.2 & 80.3 & 78.9 \\
 & GapSight$^\dagger$~\cite{zhu2026gapsight} & \xmark & 2 & 73.0 & 77.7 & 53.3 & 78.4 & 72.8 & 56.1 & 58.0 & 71.0 & 70.5 & 78.2 & 77.0 \\
 & Q-Zoom~\cite{shi2026qzoom} & \xmark & 1--2 & 81.0$^\ast$ & 82.5 & 64.2 & \textbf{84.8} & \underline{80.4} & 65.7$^\ast$ & \underline{65.2} & \underline{78.6$^\ast$} & \underline{78.0} & \underline{84.7$^\ast$} & \underline{83.0} \\
\cmidrule(l){2-15}
 & \textbf{Ours}, frozen reader & \xmark & 2 & \underline{85.0} & \underline{85.4} & \underline{64.7} & 78.6 & 76.0 & \underline{66.4} & 64.8 & 77.9 & 76.2 & 82.2 & 80.3 \\
 & \textbf{Ours} & \xmark & 2 & \textbf{87.3} & \textbf{87.4} & \textbf{67.4} & \underline{84.6} & \textbf{81.2} & \textbf{72.0} & \textbf{71.0} & \textbf{81.6} & \textbf{80.1} & \textbf{85.1} & \textbf{83.7} \\
\rowcolor{gray!8} & $\Delta$ over the best baseline & & & \gain{+6.3} & \gain{+4.9} & \gain{+3.2} & {\color{gray}-0.2} & \gain{+0.8} & \gain{+6.3} & \gain{+5.8} & \gain{+3.0} & \gain{+2.1} & \gain{+0.4} & \gain{+0.7} \\
\midrule
\multirow{8}{*}{\rotatebox[origin=c]{90}{Qwen3-VL-8B}} & Single view & -- & 1 & 67.0 & 78.1 & 50.2 & 83.2 & 73.6 & 47.9 & 53.9 & 70.4 & 71.3 & 81.8 & 80.3 \\
 & ViCrop~\cite{zhang2025vicrop} & \cmark & 2 & 69.0 & 77.0 & 51.6 & \underline{84.2} & 76.1 & 53.5 & 59.0 & 71.6 & 72.2 & 81.3 & 79.7 \\
 & AttWarp~\cite{dalal2026attwarp} & \cmark & 1 & 69.0 & 75.5 & 53.8 & 81.0 & \underline{78.3} & 47.1 & 51.9 & 71.5 & 72.1 & 83.9 & 82.4 \\
 & Ground-then-crop & \cmark & 2 & 76.0 & 82.2 & 62.9 & 84.0 & 77.4 & 61.3 & 65.1 & 76.5 & 76.6 & 83.8 & 82.4 \\
 & GapSight$^\dagger$~\cite{zhu2026gapsight} & \xmark & 2 & 79.3 & 81.3 & 56.0 & 83.2 & 75.4 & 61.5 & 62.8 & 75.0 & 74.0 & 82.7 & 81.0 \\
\cmidrule(l){2-15}
 & \textbf{Ours}, frozen reader & \xmark & 2 & \underline{87.0} & \underline{87.8} & \textbf{72.5} & 80.6 & 77.8 & \underline{70.5} & \underline{69.3} & \underline{81.1} & \underline{79.7} & \underline{85.6} & \underline{83.9} \\
 & \textbf{Ours} & \xmark & 2 & \textbf{88.7} & \textbf{89.2} & \underline{72.3} & \textbf{87.6} & \textbf{80.2} & \textbf{73.5} & \textbf{72.5} & \textbf{83.6} & \textbf{82.3} & \textbf{87.0} & \textbf{85.8} \\
\rowcolor{gray!8} & $\Delta$ over the best baseline & & & \gain{+9.4} & \gain{+7.0} & \gain{+9.4} & \gain{+3.4} & \gain{+1.9} & \gain{+12.0} & \gain{+7.4} & \gain{+7.1} & \gain{+5.7} & \gain{+3.1} & \gain{+3.4} \\
\bottomrule
\end{tabular}
\end{table*}

\noindent\textbf{Models and budgets.}
We build on Qwen3-VL-Instruct~\cite{bai2025qwen3vl} at three scales (2B, 4B, 8B).
A budget of $\B$ visual tokens per view corresponds to $\B\times32^2$ pixels (patch size 16, $2{\times}2$ token merging).
We evaluate $\B\in\{98,196,392\}$, where text-rich images lose most of their legible detail, and report $\B=784$ in supplementary Sec.~D.

\noindent\textbf{Benchmarks and metrics.}
DocVQA~\cite{mathew2021docvqa} (val, ANLS), InfographicVQA~\cite{mathew2022infographicvqa} (val, ANLS), ChartQA~\cite{masry2022chartqa} (test, relaxed accuracy), TextVQA~\cite{singh2019textvqa} (val, VQA accuracy), and V-CoT-Doc (Sec.~\ref{sec:observations}; strict match).
We evaluate a fixed random subset of 500 questions per public benchmark (seed 0) with its official metric.
V$^\ast$~\cite{wu2024vstar} and HR-Bench~\cite{wang2024hrbench}, natural images with small evidence, serve as out-of-domain tests, and the 4B comparison is repeated on the complete evaluation sets (Sec.~\ref{sec:main_results}).

\noindent\textbf{Baselines.}
We compare with four families of zooming.
\emph{Attention-guided zooming}: ViCrop~\cite{zhang2025vicrop} and AttWarp~\cite{dalal2026attwarp}, training-free methods that point with the model's own attention, re-implemented from the official code with the official default settings on the same base model and budget, with layer indices rescaled to model depth; our runs reproduce the DocVQA gains reported by AttWarp (supplementary Sec.~E).
\emph{Grounding-based zooming}: Ground-then-crop prompts the model for the bounding box of the evidence in Qwen3-VL's native grounding format, enlarges the box by its own size on each side, and answers from the budgeted view and the crop, in the spirit of Visual CoT~\cite{shao2024visualcot} and Zoom-Refine~\cite{yu2025zoomrefine}.
\emph{Learned crop routing}: we re-implement GapSight~\cite{zhu2026gapsight}, for which no code is public, and train its router on 6{,}000 images from our training sources (supplementary Sec.~E).
\emph{Learned region proposers}: Q-Zoom~\cite{shi2026qzoom}, which trains a region proposer and a gate and fine-tunes the language model, is run from its released Qwen3-VL-4B weights under our budgets.
We also report the single view; DeepEyes~\cite{zheng2026deepeyes}, an RL-trained zooming agent on a different base model, trails its own base model under our budgets (supplementary Sec.~E).

\noindent\textbf{Implementation.}
Stage~1 trains the \look token with the base model frozen for 1{,}500 steps of 8 images ($\tau{=}2$, teacher view of 1{,}568 tokens).
Its 27{,}000 questions come from the DocVQA split of Visual CoT and the training sets of ChartQA, InfographicVQA and TextVQA, without boxes and disjoint from the evaluation questions; the model's own answer is the teacher for $88$--$91\%$ of the training samples.
Stage~2 freezes the \look token and adapts the reader with GRPO for 800 steps of 16 prompts with 8 rollouts each, on 4{,}500 questions from the DocVQA split of Visual CoT, ChartQA and InfographicVQA, without boxes.
Each prompt draws its budget from $\{98,196,392\}$ tokens and each rollout its view uniformly from $\mathcal{A}$; the reward is the evaluation metric of the source dataset.
The pointer is taken from layers $13$--$16$ of 28 (2B) or $17$--$20$ of 36 (4B and 8B), about the middle of the network; both stages together take $5.1$--$8.8$\,h per scale on one H200 GPU (hyperparameters in supplementary Sec.~A).

\subsection{Main results}
\label{sec:main_results}
\ours achieves the best average accuracy at every budget and model scale (Tab.~\ref{tab:main}; examples in supplementary Sec.~F).
At 98 tokens per view, it surpasses attention-guided zooming by $20.0$--$23.8$ points and grounding-based zooming and learned crop routing by $12.0$--$14.9$ points across the three scales.
The lead over these baselines narrows as the budget grows and the information gap shrinks (Sec.~\ref{sec:analysis}), to $7.1$--$10.6$ points at 196 tokens and $3.1$--$5.3$ at 392, where AttWarp is the strongest of them for 4B and 8B.
Across benchmarks, it is largest on the small text of V-CoT-Doc and InfographicVQA ($9.4$--$15.0$ points at 196 tokens) and smallest on ChartQA and TextVQA ($1.9$--$5.6$), which the single view already reads well.
Two views of $\B$ tokens hold as many visual tokens in the answering pass as one view of $2\B$ tokens, and at 98 and 196 tokens per view \ours is $1.8$--$7.0$ points more accurate than the single view with twice the budget.

\noindent\textbf{The pointer alone.}
With the base model as the reader, \ours surpasses every baseline that leaves the language model unchanged, at every scale and budget: by $7.6$--$9.0$, $4.6$--$6.9$ and $0.6$--$1.7$ points at 98, 196 and 392 tokens.
Training the reader on the views of the fixed pointer adds $3.0$--$7.3$ points at 98 tokens.

\noindent\textbf{Learned region proposers.}
Against Q-Zoom, \ours leads by $6.3$ / $3.0$ / $0.4$ points at 98 / 196 / 392 tokens (Avg5), and by $5.8$ / $2.1$ / $0.7$ without V-CoT-Doc, whose questions are in Q-Zoom's released training data.
At 196 tokens it leads on V-CoT-Doc, DocVQA and InfographicVQA ($+6.3$ / $+4.9$ / $+3.2$), where small evidence must be found, and the two methods are within one point on ChartQA and TextVQA ($-0.2$ / $+0.8$).

\noindent\textbf{Out of domain and full sets.}
On V$^\ast$ and HR-Bench, \ours improves on the single view in all 27 model--benchmark--budget settings, by $2.1$--$14.1$ points, and on the complete evaluation sets of the 4B comparison \ours leads Q-Zoom by $5.2$ / $2.5$ / $0.4$ points at 98 / 196 / 392 tokens (supplementary Sec.~D).

\begin{table}[t]
\centering
\footnotesize
\setlength{\tabcolsep}{2.1pt}
\caption{\textbf{Other architectures} (five-benchmark average accuracy, \%) at the per-view budgets of Qwen3-VL. Best training-free: the best of ViCrop, AttWarp and Ground-then-crop in each column. SD-RPN: released weights, which exist only for Qwen2.5-VL-7B. Full results in supplementary Sec.~D. Bold: best per column.}
\label{tab:xarch_main}
\begin{tabular}{@{}l ccc ccc ccc@{}}
\toprule
& \multicolumn{3}{c}{InternVL3.5-2B} & \multicolumn{3}{c}{InternVL3.5-4B} & \multicolumn{3}{c}{Qwen2.5-VL-7B} \\ \cmidrule(lr){2-4} \cmidrule(lr){5-7} \cmidrule(l){8-10}
Method & 98 & 196 & 392 & 98 & 196 & 392 & 98 & 196 & 392 \\
\midrule
Single view & 45.1 & 48.3 & 68.5 & 46.1 & 52.8 & 71.5 & 51.1 & 72.4 & 83.1 \\
Best training-free & 50.6 & 57.0 & 73.1 & 54.1 & 60.6 & 77.3 & 58.9 & 76.1 & 84.4 \\
SD-RPN~\cite{shi2026sdrpn} & -- & -- & -- & -- & -- & -- & 64.3 & 81.5 & 87.0 \\
\textbf{Ours}, frozen reader & 66.0 & 68.1 & 77.6 & 68.6 & 71.0 & 80.2 & 75.0 & 82.3 & 86.3 \\
\textbf{Ours} & \textbf{68.4} & \textbf{69.5} & \textbf{79.7} & \textbf{69.3} & \textbf{72.2} & \textbf{80.8} & \textbf{77.8} & \textbf{84.5} & \textbf{88.1} \\
\bottomrule
\end{tabular}
\end{table}

\noindent\textbf{Other architectures.}
The same recipe, with the pointer from layers 13--16, carries over to InternVL3.5-2B / 4B~\cite{wang2025internvl35} and Qwen2.5-VL-7B~\cite{bai2025qwen25vl} (Tab.~\ref{tab:xarch_main}; supplementary Sec.~D): \ours leads every training-free baseline at every budget, by $15.2$--$18.9$ points at the smallest budget, and on Qwen2.5-VL-7B it leads SD-RPN~\cite{shi2026sdrpn}, run from its released weights, by $13.5$ / $3.0$ / $1.1$ points at 98 / 196 / 392 tokens.

\noindent\textbf{Cost.}
\ours and GapSight need two passes and generate nothing before the final answer, whereas AttWarp and Ground-then-crop first generate an answer or a box and ViCrop needs three passes (supplementary Sec.~D).
On Qwen3-VL-4B at 196 tokens, \ours is therefore faster than attention-guided and grounding-based zooming ($0.25$\,s per sample \vs $0.27$--$0.85$\,s on one H200 GPU) and within $0.03$\,s of GapSight while $10.6$ points more accurate.
Q-Zoom, which reuses the lower layers of its first pass and does not zoom on every question, takes $0.21$\,s and is $3.0$ points less accurate.

\subsection{Ablations}
\label{sec:ablation}
\begin{table}[t]
\centering
\footnotesize
\setlength{\tabcolsep}{3pt}
\caption{\textbf{Ablations} (Qwen3-VL-4B; five-benchmark average accuracy, \%, official metrics, at 98 / 196 / 392 tokens). (a)~The teacher of the \look token in stage~1 (600 steps, same data and seed, unless noted), evaluated with the frozen reader and the support window $\texttt{supp}_2$. (b)~The reader. (c)~The pointer and the window rule, with the frozen reader. Pre-answer (a, c): attention of the answer-start position, at the layers of the \look token unless noted. Bold: best per column within a block. $^\dagger$Our re-implementation. $^\ddagger$Layers 21--24, located with evidence boxes.}
\label{tab:ablation}
\begin{tabular}{l ccc}
\toprule
(a) Where to look: teacher of the pointer & 98 & 196 & 392 \\
\midrule
Pre-answer attention, 1{,}568-token view & 57.5 & 70.6 & 77.4 \\
Answer on the budgeted view & 60.5 & 73.7 & 79.2 \\
Answer on a 784-token view & 66.2 & 78.4 & \textbf{82.2} \\
Answer on a 1{,}568-token view & \textbf{67.0} & \textbf{78.9} & 81.9 \\
\quad without the correctness check & \textbf{67.0} & 78.5 & \textbf{82.2} \\
\quad 1{,}500 steps (Ours, frozen reader) & 66.4 & 77.9 & \textbf{82.2} \\
\midrule
(b) Reader & 98 & 196 & 392 \\
\midrule
Frozen reader, single view & 44.7 & 65.0 & 78.8 \\
Frozen reader, $\texttt{supp}_2$ (Ours, frozen reader) & 66.4 & 77.9 & 82.2 \\
Stage-2 reader, single view & 45.9 & 66.1 & 79.2 \\
Stage-2 reader, crops of GapSight$^\dagger$ & 60.1 & 74.8 & 82.2 \\
Stage-2 reader, $\texttt{supp}_2$ (\textbf{Ours}) & \textbf{72.0} & \textbf{81.6} & \textbf{85.1} \\
\midrule
(c) How to look: pointer, window (frozen reader) & 98 & 196 & 392 \\
\midrule
Pre-answer attention, $\texttt{adapt}_{.8}$ & 50.4 & 70.1 & 79.6 \\
\quad at its best layers$^\ddagger$, $\texttt{adapt}_{.8}$ & 54.1 & 72.1 & 80.4 \\
\quad at its best layers$^\ddagger$, $\texttt{supp}_2$ & 61.6 & 75.3 & 80.4 \\
\look token, $\texttt{adapt}_{.8}$ & 60.0 & 74.4 & 81.4 \\
\look token, $\texttt{supp}_2$ (Ours, frozen reader) & \textbf{66.4} & \textbf{77.9} & \textbf{82.2} \\
\bottomrule
\end{tabular}
\end{table}

Tab.~\ref{tab:ablation} isolates each component on Qwen3-VL-4B; blocks (a) and (c) use the frozen reader, so they measure the pointer alone.

\noindent\textbf{Where to look.}
The teacher decides what the \look token learns (Tab.~\ref{tab:ablation}a).
Distilling the answer on a 1{,}568-token view instead of the budgeted view, the same-view target of prior self-distillation~\cite{shi2026sdrpn}, adds $+6.5$ / $+5.2$ / $+2.7$ points from the resolution of the teacher alone (Sec.~\ref{sec:observations}) and raises the coverage of the token's support window at 98 tokens from $0.61$ to $0.81$ (supplementary Sec.~C); a 784-token teacher comes within $0.8$ points of it.
High resolution alone does not make a teacher: the pre-answer attention of the same 1{,}568-token view is $9.5$ / $8.3$ / $4.5$ points below the answer's attention on that view (Sec.~\ref{sec:observations}).
The teacher is robust to its own errors, since every own answer, right or wrong, stays within $0.4$ points of the checked teacher, and $600$ steps suffice: $1{,}500$ steps change accuracy by at most $1.0$ point.

\noindent\textbf{How to look.}
Both the pointer and the rule that turns it into a window matter (Tab.~\ref{tab:ablation}c).
With the adaptive window $\texttt{adapt}_{.8}$, the \look token beats the pre-answer attention of the same layers by $9.6$ / $4.3$ / $1.8$ points and that of its best layers, located with evidence boxes, by $5.9$ / $2.3$ / $1.0$; the support window of the same pointer adds another $6.4$ / $3.5$ / $0.8$, since it extracts exactly what the token is trained to reproduce, and with this window the token leads the best-layer pre-answer attention by $4.8$ / $2.6$ / $1.8$.

\noindent\textbf{Reader.}
With the frozen reader, the support window adds $21.7$ / $12.9$ / $3.4$ points to the single view (Tab.~\ref{tab:ablation}b): stage~1 leaves the model untouched, so pointing costs no reading accuracy (\cf Sec.~\ref{sec:look}), and the base model reads the windows of a pointer it was never trained with.
Stage~2 adds $5.6$ / $3.7$ / $2.9$ points on the same windows while changing single-view accuracy by at most $1.2$ points, so the reader learns to use the window rather than to read better in general.
The same reader scores $11.9$ / $6.8$ / $2.9$ points higher on our windows than on the crops of GapSight.

\subsection{Analysis}
\label{sec:analysis}
\begin{figure}[t]
  \centering
  \includegraphics[width=\linewidth]{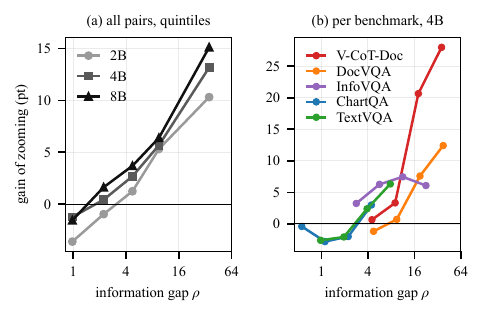}
  \caption{\textbf{The information gap governs the value of looking again.} Accuracy gain of a zoomed view (a $0.7$ crop placed by the pre-answer attention of the base model) over the single view, against the information gap $\rho$ (native over budget pixels). (a)~Quintiles of $\rho$ over 9{,}200 image--budget pairs per model (2{,}300 questions, 4 budgets). (b)~Per benchmark (4B): the same images at budgets from 784 to 98 tokens.}
  \label{fig:infogap}
\end{figure}

\noindent\textbf{Where looking again helps.}
To see where zooming pays off, we compare a zoomed view, a $0.7$ crop placed by the pre-answer attention of the base model, with the single view on all evaluation questions at four budgets.
The gain rises monotonically with the information gap $\rho$ (Fig.~\ref{fig:infogap}a): zooming costs $1.3$--$3.6$ points when $\rho<1.6$ and brings $10.3$--$15.1$ points when $\rho>18$, and within each benchmark, where only the budget changes, it is larger at 98 than at 784 tokens (Fig.~\ref{fig:infogap}b).
The best fixed view shifts with $\rho$ as well: on 6{,}000 training images, gentle views (a foveated warp or a $0.7$ crop) are best for $\rho<2.2$ and strong magnification (a $0.5$ crop or square window) at larger gaps, and no fixed view is best for every sample, since a zoomed view of relative area $A$ adds at most a factor $\min(\rho,1/A)$ of pixel density (supplementary Sec.~D).
\ours therefore takes the extent of its window from the pointer and keeps the budgeted view as the first view, so that a window that adds little detail costs little.

\noindent\textbf{What the \look token extracts.}
On V-CoT-Doc (Fig.~\ref{fig:attn}), the \look token places $0.213$ / $0.218$ / $0.200$ of its attention on the evidence (2B / 4B / 8B, 196 tokens), twice the pre-answer attention of the same layers and $1.2$--$1.4{\times}$ that of the best top layers.
Its support window covers $0.83$ / $0.82$ / $0.82$ of the evidence at 196 tokens with a median area of a sixth of the image, and although the \look prefill runs through the adapted reader, stage~2 leaves the pointer in place: the windows before and after it have a median IoU of $1.00$ (supplementary Sec.~C).

\noindent\textbf{Pointing \vs naming the region.}
At 98 tokens, the region selected by \ours covers $0.73$--$0.78$ of the evidence area on V-CoT-Doc, the box named by Ground-then-crop $0.34$--$0.41$, and the first crop of DeepEyes-7B $0.15$ (supplementary Sec.~D).
The difference narrows as the budget grows; with the same frozen reader on 4B, our window is $8.9$ / $7.0$ / $1.9$ points more accurate than the named box at 98 / 196 / 392 tokens: text that cannot be read is hard to name, but its location can still be pointed at.
Against Q-Zoom, our window covers at least half of the evidence more often ($81$ \vs $69\%$ at 98 tokens and $85$ \vs $73\%$ at 196, counting for Q-Zoom the questions on which it zooms) and is read at least as accurately when it does ($81.8$ \vs $81.3$ and $92.1$ \vs $88.0$; supplementary Sec.~D).

\section{Conclusion}
\label{sec:conclusion}
Before an MLLM answers, its attention localizes the evidence only in the top layers, and its answer localizes the evidence precisely only on a view it can read.
\ours distills the answer's localization on a high-resolution view into a \look token on the budgeted view with the model frozen, zooms into the support of the token's pointer, and then trains the reader on the views that the pointer selects.
The pointer alone surpasses every baseline that leaves the language model unchanged at 98 to 392 tokens, the trained reader widens the margin at every budget, and the method outperforms learned region proposers and carries over to InternVL3.5 and Qwen2.5-VL.
\ours zooms on every question, and its lead narrows as the budget grows and the information gap closes; on natural images, naming the object remains competitive on some benchmarks (supplementary Sec.~G).
Deciding when to zoom, looking at several regions and video are natural next steps.

{
    \small
    \bibliographystyle{ieeenat_fullname}
    \bibliography{refs}
}

\begin{strip}\centering\vspace{0.5em}{\Large\bf \papertitle\par}\vspace{0.4em}{\large Supplementary Material\par}\vspace{0.3em}\end{strip}
\setcounter{table}{0}\setcounter{figure}{0}\setcounter{equation}{0}
\renewcommand{\thetable}{S\arabic{table}}\renewcommand{\theHtable}{S\arabic{table}}
\renewcommand{\thefigure}{S\arabic{figure}}\renewcommand{\theHfigure}{S\arabic{figure}}
\renewcommand{\theequation}{S\arabic{equation}}\renewcommand{\theHequation}{S\arabic{equation}}
\renewcommand{\theproposition}{S\arabic{proposition}}
\renewcommand{\topfraction}{0.9}\renewcommand{\bottomfraction}{0.5}\renewcommand{\textfraction}{0.1}
\renewcommand{\floatpagefraction}{0.75}\renewcommand{\dbltopfraction}{1.0}\renewcommand{\dblfloatpagefraction}{0.75}
\setcounter{topnumber}{3}\setcounter{dbltopnumber}{3}\setcounter{totalnumber}{5}
\makeatletter   %
\setlength{\@fptop}{0pt}\setlength{\@fpsep}{14pt plus 2pt}\setlength{\@fpbot}{0pt plus 1fil}
\setlength{\@dblfptop}{0pt}\setlength{\@dblfpsep}{6pt plus 2pt}\setlength{\@dblfpbot}{0pt plus 1fil}
\makeatother
\makeatletter
\let\lz@table\table \let\lz@endtable\endtable \newif\iflz@wide
\renewcommand{\table}[1][t]{\ifnum\value{table}=16 \global\lz@widetrue \@nameuse{table*}[#1]\else \global\lz@widefalse \lz@table[#1]\fi}
\renewcommand{\endtable}{\iflz@wide \@nameuse{endtable*}\global\lz@widefalse \else \lz@endtable\fi}
\makeatother
\appendix

\noindent\textbf{Roadmap.}
Sec.~\ref{sec:supp_impl} gives the implementation details of Secs.~\ref{sec:method} and~\ref{sec:exp_setup}.
Secs.~\ref{sec:supp_obs}, \ref{sec:supp_method} and~\ref{sec:supp_exp} supplement Secs.~\ref{sec:observations}, \ref{sec:method} and~\ref{sec:experiments} of the main paper (analysis, method and experiments) and follow their order.
Sec.~\ref{sec:supp_baselines} details the baselines, Sec.~\ref{sec:supp_qual} shows qualitative results and Sec.~\ref{sec:supp_limits} discusses limitations.

\section{Implementation details}
\label{sec:supp_impl}
This section lists what is needed to reproduce \ours and the baselines: prompts, the views of $\mathcal{A}$, data, hyperparameters, evaluation protocol and compute.

\noindent\textbf{Prompts.}
Every question is followed by ``Answer the question using a single word or phrase.''
When a zoomed view is added (\ours, Ground-then-crop and GapSight), the prompt is prefixed with ``The second image is a magnified view of the relevant region of the first image.''
ViCrop and AttWarp use their official input formats: the original image and the crop without extra text (ViCrop), and the warped image alone (AttWarp).
Ground-then-crop first asks ``Locate the region of the image that contains the information needed to answer the question: ``$\langle q\rangle$''. Output its bounding box in JSON format as \{"bbox\_2d": [x1, y1, x2, y2]\}.'' and parses the box in Qwen3-VL's 0--1000 relative coordinates; when no valid box is returned, it answers from the single view.
V$^\ast$ and HR-Bench are multiple-choice: the question lists the options and ends with an instruction to answer with the option letter, and no short-answer instruction is added.

\noindent\textbf{Views.}
Tab.~\ref{tab:views} defines the sixteen views of $\mathcal{A}$ on which the reader is trained, one drawn uniformly for each rollout (Sec.~\ref{sec:reader}).
Besides \texttt{none}, ten views are computed on the smoothed pointer $G_\sigma*\gl$, and five are variants of the support and peak windows on the unsmoothed pointer or after zeroing the four corner cells.
At test time \ours uses the support window $\texttt{supp}_2$ on the smoothed pointer.
\begin{table*}[t]
\centering
\footnotesize
\setlength{\tabcolsep}{4pt}
\caption{\textbf{The sixteen views of $\mathcal{A}$} on which the reader is trained in stage~2 (one drawn uniformly for each rollout). $h\times w$: token grid of the budgeted view ($N=hw$ cells); $W\times H$: image size in pixels. Pointer: \emph{smoothed} is $G_\sigma*\gl$ ($\sigma=1$ cell), \emph{unsmoothed} is $\gl$ of Eq.~\eqref{eq:pointer}; \emph{de-sinked} zeroes the four corner cells, which absorb attention sinks, and renormalizes. Every view except \texttt{none} is given to the reader as a second image after the budgeted view; the box views crop the image to their box and resize the crop (bicubic, aspect ratio kept) to $\B\times32^2$ pixels. At test time, \ours uses $\texttt{supp}_2$ on the smoothed pointer.}
\label{tab:views}
\begin{tabular}{@{}l l p{0.58\textwidth}@{}}
\toprule
View & Pointer & Definition \\
\midrule
\texttt{none} & -- & No second view: the reader answers from the budgeted view alone. \\
$\texttt{crop}_{.7}$ & smoothed & Window of $.7h\times.7w$ cells (rounded) with the largest pointer mass. \\
$\texttt{crop}_{.5}$ & smoothed & As $\texttt{crop}_{.7}$, with $.5h\times.5w$ cells. \\
$\texttt{square}_{.5}$ & smoothed & Square of side $.5\min(W,H)$ pixels: the window of this size with the largest pointer mass, re-centered on the pointer centroid inside it (one mean-shift step) and kept inside the image. \\
$\texttt{square}_{.3}$ & smoothed & As $\texttt{square}_{.5}$, with side $.3\min(W,H)$. \\
\texttt{fovea} & smoothed & The whole image, resampled separably along each axis by the inverse CDF of the density $0.5\cdot$(pointer marginal)$\,+\,0.5\cdot$uniform, at the size of the budgeted view. \\
$\texttt{adapt}_{.8}$ & smoothed & Central $80\%$ of the pointer mass along each axis (marginal quantiles $.1$ and $.9$); each side at least $12\%$ of $\min(W,H)$. \\
$\texttt{supp}_2$ & smoothed & Bounding box of the 4-connected component of cells with pointer mass at least $2/N$ that contains the peak cell (the peak cell alone if it is below this threshold). \\
$\texttt{supp}_3$ & smoothed & As $\texttt{supp}_2$, with threshold $3/N$. \\
$\texttt{supp}_2$, de-sinked & smoothed, de-sinked & $\texttt{supp}_2$ after de-sinking. \\
$\texttt{peak}_{.5}$ & smoothed & Box grown from the peak cell one row or column at a time, on the side that adds the most mass, until it holds half of the pointer mass; then enlarged about its center, inside the image, to an area of at least $\B\times32^2$ pixels. \\
$\texttt{gap}_{2}$ & smoothed & Window of area $\min(WH,\,2\B\times32^2)$ pixels with the aspect ratio of the image, centered on the pointer centroid inside the half-mass box of $\texttt{peak}_{.5}$ (before its enlargement) and kept inside the image. \\
$\texttt{supp}_2$, unsmoothed & unsmoothed & $\texttt{supp}_2$ on the unsmoothed pointer. \\
$\texttt{supp}_3$, unsmoothed & unsmoothed & $\texttt{supp}_3$ on the unsmoothed pointer. \\
$\texttt{supp}_2$, unsmoothed, de-sinked & unsmoothed, de-sinked & $\texttt{supp}_2$ on the unsmoothed pointer after de-sinking. \\
$\texttt{peak}_{.5}$, unsmoothed & unsmoothed & $\texttt{peak}_{.5}$ on the unsmoothed pointer. \\
\bottomrule
\end{tabular}
\end{table*}

\noindent\textbf{Data and splits.}
V-CoT-Doc is built from the document (DocVQA) split of Visual CoT~\cite{shao2024visualcot}.
In file order, we keep the questions whose image is available and that have at least one evidence box; V-CoT-Doc is entries 1{,}500 to 1{,}799 of this list (300 questions).
Stage~1 trains on the first 15{,}000 entries outside 1{,}500 to 2{,}399, and stage~2 on entries 0 to 1{,}499, so no V-CoT-Doc question is used for training.
Entries 1{,}800 to 2{,}099 are the further held-out questions on which we analyze the teachers (Tabs.~\ref{tab:teacher_res} and~\ref{tab:crsd_controls}), so these analyses do not use the test questions either.
The other training questions are fixed random draws from the official training splits of ChartQA, InfographicVQA and TextVQA.
The public benchmarks are evaluated on fixed random subsets of 500 questions (seed 0) of DocVQA val, InfographicVQA val, ChartQA test and TextVQA val; Sec.~\ref{sec:supp_exp} repeats the 4B comparison on the complete sets.

\noindent\textbf{Hyperparameters and optimization.}
Tab.~\ref{tab:hparams} lists the hyperparameters of the two stages and of evaluation.
Stage~1 minimizes the KL divergence of Sec.~\ref{sec:distill} with respect to the \look parameters only.
Stage~2 is on-policy GRPO without a reference model: the rewards of the $G$ rollouts of a prompt are normalized by their mean and standard deviation, each rollout's mean token log-likelihood is weighted by its advantage, and one gradient step is taken per batch of rollouts, so no importance ratio, ratio clipping or KL penalty is used.
Prompts whose rollouts all receive the same reward contribute no gradient.

\begin{table*}[t]
\centering
\footnotesize
\setlength{\tabcolsep}{3pt}
\caption{\textbf{Hyperparameters} (identical for the three model scales unless stated; $\B$: visual tokens per view).}
\label{tab:hparams}
\begin{tabular}[t]{@{}p{2.5cm}p{5.6cm}@{}}
\toprule
\multicolumn{2}{@{}l}{\emph{Stage 1: pointer (base model frozen)}} \\
\midrule
\look token & unused special token after the assistant prefix; learned embedding offset, initialized so that the token starts as the word ``look'' \\
\look LoRA & rank 16, $\alpha{=}32$, attention and MLP projections of layers $0$--$\ell_{\max}$, at the \look position only \\
$\ell_{\max}$ / pointer layers $\Lambda$ & 16 / 13--16 (2B); 20 / 17--20 (4B, 8B) \\
Teacher view & 1{,}568 tokens \\
Teacher answer & own greedy answer (at most 48 tokens) if its score is at least $0.5$, else the ground truth \\
Target & answer attention pooled onto the student grid, corners zeroed, on the peak component at $\tau{=}2$ \\
Loss & $\mathrm{KL}(\hat g\,\|\,\gl)$ on the unsmoothed student pointer, averaged over the images of a step \\
Student budgets & uniform over $\{98, 196, 392\}$, per image \\
Steps / images per step & 1{,}500 / 8 \\
Optimizer & AdamW, learning rate $10^{-4}$, no weight decay, constant, no warm-up; gradient-norm clip 1.0 \\
Training questions & 15k Visual CoT (DocVQA); 4k each ChartQA, InfographicVQA, TextVQA \\
Seed & 0 \\
\midrule
\multicolumn{2}{@{}l}{\emph{Evaluation}} \\
\midrule
Views & resized (bicubic, aspect ratio kept, never enlarged) to at most $\B\times32^2$ pixels \\
Window & $\texttt{supp}_2$ on $G_\sigma*\gl$, $\sigma=1$ cell \\
Decoding & greedy, at most 48 new tokens \\
\bottomrule
\end{tabular}
\hspace{1.2em}
\begin{tabular}[t]{@{}p{2.5cm}p{5.6cm}@{}}
\toprule
\multicolumn{2}{@{}l}{\emph{Stage 2: reader (\look token frozen)}} \\
\midrule
Training questions & 1{,}500 each from Visual CoT (DocVQA), ChartQA, InfographicVQA \\
Steps / prompts $P$ / rollouts $G$ & 800 / 16 / 8 \\
Training budgets & uniform over $\{98, 196, 392\}$, per prompt \\
View of each rollout & uniform over the 16 views of Tab.~\ref{tab:views} \\
Rollout sampling & temperature 1.0, top-$p$ 0.95, at most 48 new tokens \\
Reward & evaluation metric of the source dataset: relaxed accuracy (ChartQA), ANLS (InfographicVQA), strict match (Visual CoT) \\
Advantage & reward minus the mean over the $G$ rollouts of the prompt, divided by their standard deviation \\
Objective & advantage-weighted mean token log-likelihood, averaged over the $G$ rollouts and summed over the $P$ prompts; one on-policy step per batch; no KL penalty \\
Reader LoRA & rank 16, $\alpha{=}32$, dropout 0.05, attention and MLP projections of all language-model layers \\
Optimizer & AdamW, learning rate $10^{-5}$, no weight decay, constant, no warm-up; gradient-norm clip 1.0 \\
Seed & 0 for steps 1--400; training resumes from step 400 with seed 1 and a new optimizer state \\
\bottomrule
\end{tabular}
\end{table*}

\noindent\textbf{Evaluation protocol.}
All methods decode greedily with at most 48 new tokens.
Each view is resized to at most $\B\times32^2$ pixels and encoded by the Qwen3-VL processor into tokens of $32\times32$ pixels, with each side rounded to a multiple of 32 pixels.
DocVQA and InfographicVQA use ANLS, ChartQA relaxed accuracy and TextVQA VQA accuracy, with the official scorers.
V-CoT-Doc uses strict match: after lowercasing and removing punctuation, every word of the ground truth appears in the answer (for yes / no questions, the first yes or no of the answer must match), and the answer has at most three words more than the ground truth.
V$^\ast$ and HR-Bench are scored by the option letter.

\noindent\textbf{Compute.}
On one H200 GPU shared with other jobs, stage~1 takes 2.3\,h (2B), 4.4\,h (4B) and 4.6\,h (8B); each step answers 8 questions on the teacher view and back-propagates through the \look parameters only.
Stage~2 takes 2.8\,h (2B), 3.9\,h (4B) and 4.2\,h (8B).

\section{Supplements to the analysis}
\label{sec:supp_obs}
This section gives the per-layer and per-resolution evidence behind the analysis of Sec.~\ref{sec:observations}.

\noindent\textbf{Per-layer evidence mass (Sec.~\ref{sec:observations}).}
Fig.~\ref{fig:layers} computes the evidence mass of Fig.~\ref{fig:attn} layer by layer.
The attention of the answer tokens localizes the evidence from the middle layers on, whereas the pre-answer attention does so only in the top layers, and the layers where it peaks differ across scales (shaded in gray: the best band of four layers, 17--20 of 28 for 2B, 21--24 and 28--31 of 36 for 4B and 8B).
The pointer of the \look token is taken from a fixed band of middle layers (shaded in red); its evidence mass, aggregated over this band (red line), exceeds both curves at every layer.
\begin{figure*}[t]
  \centering
  \includegraphics[width=\textwidth]{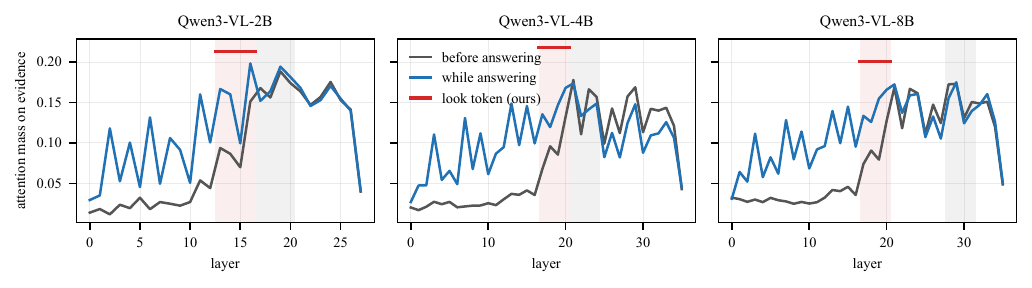}
  \caption{\textbf{Per-layer evidence mass} (V-CoT-Doc, 196 tokens; each layer averaged over heads, normalized and smoothed on its own). Red band: layers from which the pointer of the \look token is taken, with its evidence mass aggregated over the band (red line, as in Fig.~\ref{fig:attn}); gray band: best four layers of the pre-answer attention, located with the evidence boxes.}
  \label{fig:layers}
\end{figure*}

\noindent\textbf{Teacher resolution (Sec.~\ref{sec:observations}).}
Tab.~\ref{tab:teacher_res} quantifies, for the three scales, how precisely the answer localizes the evidence at each resolution.
The answer-time attention of the base model, teacher-forced on the ground truth, is pooled onto the 98-token grid, and its support window is scored against the evidence boxes of 300 further held-out Visual CoT questions, used neither in training nor in V-CoT-Doc.
The window covers about half of the evidence when the answer is produced on the 98-token budgeted view and about four fifths on the 1{,}568-token view, while its median area grows only from $6$ to $9\%$ of the image.
\begin{table}[t]
\centering
\footnotesize
\setlength{\tabcolsep}{3pt}
\caption{\textbf{The answer localizes the evidence sharply only on a legible view} (Sec.~\ref{sec:observations}; base models, no training). Answer-time attention of the ground-truth answer on views of 98 to 1{,}568 tokens, pooled onto the 98-token grid; coverage of the evidence by its support window $\texttt{supp}_2$ (fraction of the evidence area inside the window) and the window's median area (fraction of the image), on 300 further held-out Visual CoT questions, used neither in training nor in V-CoT-Doc. Bold: highest coverage per model.}
\label{tab:teacher_res}
\begin{tabular}{l ccc ccc}
\toprule
& \multicolumn{3}{c}{Coverage} & \multicolumn{3}{c}{Window area} \\ \cmidrule(lr){2-4} \cmidrule(lr){5-7}
Answer produced on & 2B & 4B & 8B & 2B & 4B & 8B \\
\midrule
98 tokens (budgeted view) & 0.51 & 0.54 & 0.47 & 0.06 & 0.06 & 0.06 \\
392 tokens & 0.68 & 0.70 & 0.66 & 0.08 & 0.09 & 0.08 \\
784 tokens & 0.78 & 0.76 & 0.70 & 0.09 & 0.09 & 0.09 \\
1{,}568 tokens & \textbf{0.82} & \textbf{0.83} & \textbf{0.79} & 0.09 & 0.09 & 0.09 \\
\bottomrule
\end{tabular}
\end{table}

\section{Supplements to the method}
\label{sec:supp_method}
This section supports the design of Sec.~\ref{sec:method}: the reading cost of supervising the shared attention, what the \look token learns from each teacher, why stage~2 leaves the pointer in place, and three properties of the design in idealized settings.

\noindent\textbf{Supervising the shared attention (Sec.~\ref{sec:look}).}
The reading cost of supervising the shared attention comes from development runs on Qwen3-VL-2B that precede the final recipe (Tab.~\ref{tab:look_training}).
In these runs a \look token is distilled from the answer-time attention of the budgeted view while the reader is fine-tuned with GRPO.
The shared-attention variant is trained in exactly the same way, except that the distillation loss is applied to the pre-answer attention of the last prompt position and its gradients reach the reader.
Reading accuracy in panel~(a) is the single-view accuracy on V-CoT-Doc, DocVQA, InfographicVQA, ChartQA and TextVQA at the budgets where all variants were evaluated (11 benchmark--budget settings), compared with the same RL fine-tuning without pointer losses.
Panel~(b) changes one element at a time.
Sharpening the target raises the evidence mass from $0.152$ to $0.229$ and helps the adaptive window.
Restricting the targets to correct answers matters little, since the answer-time attention is on average as accurate as that of the ground truth (Sec.~\ref{sec:observations}).
Supervising the shared attention instead of a dedicated token reaches a similar evidence mass but lowers single-view accuracy and the accuracy of every zoomed view (the adaptive window in Tab.~\ref{tab:look_training}, and also the foveated and square views).
\ours keeps the dedicated token, replaces the same-view target by the pooled high-resolution target of Sec.~\ref{sec:distill}, and trains the token on the frozen model.
\begin{table}[t]
\centering
\footnotesize
\setlength{\tabcolsep}{3pt}
\caption{\textbf{Development runs of pointer training} (Qwen3-VL-2B; not the final recipe). These runs provide the evidence for the reading cost of supervising the shared attention (Sec.~\ref{sec:look}). The \look token is distilled from the answer-time attention of the budgeted view itself while the reader is fine-tuned with GRPO; the final recipe instead uses a 1{,}568-token teacher, the support target ($\tau=2$) and the frozen base model (Sec.~\ref{sec:distill}). Evidence mass of the pointer on V-CoT-Doc at $\B=196$ / 392 visual tokens per view. (a)~Early runs in which the reader answers from the zoomed view alone; training signals accumulate from top to bottom, and reading is the change of single-view accuracy against the same RL fine-tuning without pointer losses (mean over 11 benchmark--budget settings). (b)~Runs with the budgeted and the zoomed view; each row changes one element of training, and accuracy is the five-benchmark average over 98 / 196 / 392 tokens for single-view reading and for the adaptive window $\texttt{adapt}_{.8}$. Sharpened: the target is the squared answer attention; correctness check: only correct sampled answers provide targets; shared attention: the pre-answer attention of the last prompt position is supervised instead of a \look token, and its gradients reach the reader. The two panels are separate runs. Bold: best per column within a panel.}
\label{tab:look_training}
\begin{tabular}{l cc c}
\toprule
& \multicolumn{2}{c}{Evidence mass} & Reading \\ \cmidrule(lr){2-3}
(a) Zoomed view alone & 196 & 392 & $\Delta$ acc. \\
\midrule
\look token, untrained & .068 & .057 & -- \\
\quad + reward of the selected views & .137 & .145 & $+0.0$ \\
\quad + answer-attention distillation & .152 & .161 & $0.0$ \\
\quad + sharpened target & \textbf{.221} & \textbf{.257} & $\mathbf{+0.1}$ \\
\midrule
Shared attention, supervised & .217 & .247 & $-0.7$ \\
\bottomrule
\end{tabular}

\vspace{0.6em}
\begin{tabular}{l cc cc}
\toprule
& \multicolumn{2}{c}{Evidence mass} & \multicolumn{2}{c}{Accuracy} \\ \cmidrule(lr){2-3} \cmidrule(lr){4-5}
(b) Budgeted and zoomed view & 196 & 392 & Single & $\texttt{adapt}_{.8}$ \\
\midrule
\look token, sharpened target & \textbf{.229} & \textbf{.266} & 58.3 & \textbf{69.0} \\
\quad w/o sharpening & .152 & .164 & \textbf{58.4} & 65.3 \\
\quad w/o correctness check & .224 & .262 & 58.3 & 68.7 \\
\midrule
Shared attention, supervised & .219 & .251 & 57.3 & 66.9 \\
\bottomrule
\end{tabular}
\end{table}

\noindent\textbf{What the \look token learns from each teacher.}
Tab.~\ref{tab:crsd_controls} trains the \look token for 600 steps of stage~1 on Qwen3-VL-4B with each teacher and measures its support window on the same held-out questions, at the three student budgets, together with the accuracy of the same runs.
With a teacher at the student's own resolution, the support window covers $0.59$--$0.61$ of the evidence; a 784-token teacher raises this to $0.79$--$0.86$ and the 1{,}568-token teacher to $0.81$--$0.86$.
Two further teachers isolate what the signal is.
The pre-answer attention of the last prompt position on the same 1{,}568-token view gives $0.35$--$0.48$, below the answer on the budgeted view.
The model's own answers used without a correctness check give coverage within $0.01$ of the checked teacher, so the pointer can also be trained without answer labels.
\begin{table}[t]
\centering
\footnotesize
\setlength{\tabcolsep}{3pt}
\caption{\textbf{What the \look token learns from each teacher} (Qwen3-VL-4B; stage~1 on the frozen base model, 600 steps with the same data and seed; Ours, frozen reader: the teacher of the fourth row trained for the full 1{,}500 steps). Coverage: fraction of the evidence area inside the support window $\texttt{supp}_2$ of the \look pointer, on 300 further held-out Visual CoT questions, used neither in training nor in V-CoT-Doc. Accuracy: five-benchmark average with the frozen base reader and $\texttt{supp}_2$ (the runs of Tab.~\ref{tab:ablation}a). Columns: $\B$ visual tokens per view of the student. \emph{Pre-answer} distills the attention of the last prompt position on the 1{,}568-token view; \emph{answer} distills the answer-time attention of the model's own answer when it is correct and of the ground truth otherwise, on the view of the row; \emph{w/o correctness check} uses the model's own answer without checking it. Bold: best per column.}
\label{tab:crsd_controls}
\begin{tabular}{l ccc ccc}
\toprule
& \multicolumn{3}{c}{Coverage} & \multicolumn{3}{c}{Accuracy} \\ \cmidrule(lr){2-4} \cmidrule(lr){5-7}
Teacher & 98 & 196 & 392 & 98 & 196 & 392 \\
\midrule
Pre-answer, 1{,}568 tokens & 0.48 & 0.45 & 0.35 & 57.5 & 70.6 & 77.4 \\
Answer, budgeted view & 0.61 & 0.60 & 0.59 & 60.5 & 73.7 & 79.2 \\
Answer, 784 tokens & 0.79 & 0.85 & 0.86 & 66.2 & 78.4 & \textbf{82.2} \\
Answer, 1{,}568 tokens & \textbf{0.81} & 0.85 & 0.86 & \textbf{67.0} & \textbf{78.9} & 81.9 \\
\quad w/o correctness check & 0.80 & 0.85 & 0.86 & \textbf{67.0} & 78.5 & \textbf{82.2} \\
\midrule
\textbf{Ours}, frozen reader & \textbf{0.81} & \textbf{0.86} & \textbf{0.87} & 66.4 & 77.9 & \textbf{82.2} \\
\bottomrule
\end{tabular}
\end{table}

\noindent\textbf{Does stage~2 move the pointer?}
The prefill of the \look token runs through the adapted reader, so stage~2 could move the pointer although the parameters of the token are frozen.
It does not: at every scale, benchmark and budget, the windows before and after stage~2 have a median IoU of $1.00$, and $82$--$100\%$ of them overlap with an IoU above $0.5$.
The evidence mass of the token on V-CoT-Doc is $0.205$ / $0.215$ / $0.198$ before and $0.213$ / $0.218$ / $0.200$ after stage~2 (2B / 4B / 8B, 196 tokens).

\noindent\textbf{Properties of the design.}
Propositions~\ref{prop:conc}--\ref{prop:iso} make explicit, in idealized settings, why the window adapts to the pointer, what the distillation estimates, and why pointing leaves reading unchanged.

\begin{proposition}[Concentration sets the magnification of the adaptive window]
\label{prop:conc}
Consider one image axis, normalized to $[0,1]$, with evidence on the interval $E$ of length $w$ centered on $1/2$, and a pointer marginal that places mass $q$ uniformly on $E$ and $1-q$ uniformly on $[0,1]$.
The window spanning the central fraction $p$ of the pointer mass has length
\begin{equation}
W(q)=\begin{cases}\dfrac{p-q}{1-q}, & q<q^\star,\\[2mm] \dfrac{pw}{q+w(1-q)}, & q\ge q^\star,\end{cases}
\end{equation}
with $q^\star=(p-w)/(1-w)$; it decreases monotonically from $W(0)=p$ through $W(q^\star)=w$ to $W(1)=pw$.
\end{proposition}
\noindent\emph{Proof.}
The symmetric interval $[\tfrac12-t,\tfrac12+t]$ carries mass $M(t)=q\min(1,2t/w)+(1-q)\,2t$.
Solving $M(t)=p$ with $2t\ge w$ gives the first case, valid when $q<q^\star$, and with $2t<w$ the second; both expressions equal $w$ at $q^\star$ and decrease in $q$.

\noindent A diffuse pointer ($q\to0$) yields a window spanning a fraction $p$ of the axis, which barely magnifies for $p=0.8$; a pointer with $q=q^\star$ yields the smallest window that still contains the evidence, magnifying it by $1/w$ along the axis; a more concentrated pointer clips the evidence to its central fraction $p/(q+w(1-q))$.
Two pointers that place their peak at the same location can thus produce very different windows.
This is why the \look token, which concentrates twice as much mass on the evidence as the pre-answer attention of the same layers (Sec.~\ref{sec:analysis}), gains on the adaptive window (Tab.~\ref{tab:ablation}c); it is also consistent with the clipped window in Fig.~\ref{fig:qual_failure}.

\begin{proposition}[Distillation estimates the expected high-resolution target]
\label{prop:distill}
Let $x$ denote the prompt (budgeted view and question) and $\hat g$ the target of Eq.~\eqref{eq:target}, which depends on the high-resolution view and on the answer and is therefore random given $x$.
Over all maps $g$ on the image grid, the expected loss $\mathbb{E}_{\hat g\mid x}\,\mathrm{KL}(\hat g\,\|\,g)$ is minimized by $g^\star(x)=\mathbb{E}[\hat g\mid x]$, and the excess loss of any $g$ equals $\mathrm{KL}(g^\star\,\|\,g)$.
\end{proposition}
\noindent\emph{Proof.}
$\mathbb{E}\,\mathrm{KL}(\hat g\,\|\,g)=\mathbb{E}\sum_i\hat g_i\log\hat g_i-\sum_i g^\star_i\log g_i$, and the second term is the cross-entropy $H(g^\star)+\mathrm{KL}(g^\star\,\|\,g)$; only $\mathrm{KL}(g^\star\,\|\,g)$ depends on $g$.

\noindent The \look token is thus trained to output, from the budgeted view alone, where the answer attends on average on a legible view.
It extracts this localization without the high-resolution view and without decoding an answer at test time.

\begin{proposition}[Pointing does not change how a view is read]
\label{prop:iso}
Let $\theta_r$ be the parameters of the reader and $\theta_\ell$ those of the \look token (embedding offset and position-specific LoRA).
(a)~In the prefill on $[v,q,\look]$, the hidden states of all positions before \look are independent of $\theta_\ell$.
(b)~For a given view $z$, the answer distribution $p(y\mid v,z,q)$ is independent of $\theta_\ell$.
(c)~Stage~1 leaves $\theta_r$ unchanged, and stage~2 leaves $\theta_\ell$ unchanged.
\end{proposition}
\noindent\emph{Proof.}
(a)~Under the causal mask no earlier position attends to \look, the last position, and the LoRA acts only at that position.
(b)~\look is absent from the answering pass.
(c)~Stage~1 back-propagates only into $\theta_\ell$, and stage~2 only into $\theta_r$.

\noindent Training the pointer therefore affects reading only through the choice of the view.
Supervising the shared pre-answer attention instead moves $\theta_r$ and changes the answer distribution for every view, which is the reading cost measured in Sec.~\ref{sec:look}.

\section{Supplements to the experiments}
\label{sec:supp_exp}
This section completes the results of Sec.~\ref{sec:experiments}: all budgets and benchmarks, the complete evaluation sets, out-of-domain images, other architectures, cost, extended ablations, the information gap and analyses of the zoomed view.

\noindent\textbf{Full results.}
Fig.~\ref{fig:curves_suppl} plots the five-benchmark average against the budget, and Tabs.~\ref{tab:full_2b}--\ref{tab:full_8b} list all results, including 784 tokens, where the information gap is small; \ours still leads the best baseline there by $2.2$ / $2.2$ / $1.0$ points (2B / 4B / 8B).
\begin{figure*}[t]
  \centering
  \includegraphics[width=\textwidth]{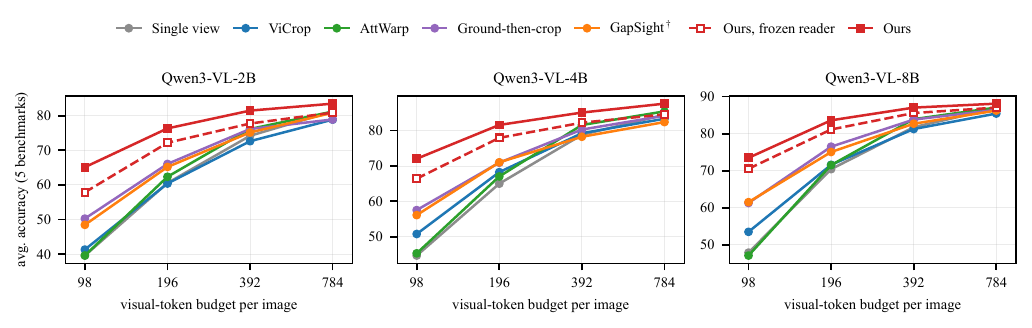}
  \caption{\textbf{Accuracy across visual-token budgets} (average over five benchmarks), including 784 tokens. The dashed line is Ours, frozen reader (stage~1 only). GapSight$^\dagger$: our re-implementation.}
  \label{fig:curves_suppl}
\end{figure*}
\begin{table}[t]
\centering
\footnotesize
\setlength{\tabcolsep}{1.9pt}
\caption{\textbf{Full results of Qwen3-VL-2B} at $\B=98$ to $784$ visual tokens per view (accuracy, \%, official metrics). Avg5: five-benchmark average; Avg4: without V-CoT-Doc. Ours, frozen reader: stage~1 only; $^\dagger$our re-implementation (Sec.~\ref{sec:supp_gapsight}). 784 tokens is the boundary regime with a small information gap. Bold: best; underlined: second best per budget and column.}
\label{tab:full_2b}
\begin{tabular}{l ccccc cc}
\toprule
Method & V-CoT & Doc & Info & Chart & Text & Avg5 & Avg4 \\
\midrule
\multicolumn{8}{l}{\emph{98 tokens}} \\
Single view & 18.0 & 41.7 & 27.0 & 57.0 & 54.0 & 39.5 & 44.9 \\
ViCrop~\cite{zhang2025vicrop} & 22.7 & 42.2 & 28.3 & 53.8 & 59.4 & 41.3 & 45.9 \\
AttWarp~\cite{dalal2026attwarp} & 16.7 & 32.7 & 30.3 & \underline{57.4} & 60.7 & 39.6 & 45.3 \\
Ground-then-crop & 45.0 & 52.8 & 33.8 & 55.0 & 64.5 & 50.2 & 51.5 \\
GapSight$^\dagger$~\cite{zhu2026gapsight} & 50.3 & 55.2 & 29.7 & 55.2 & 51.7 & 48.4 & 48.0 \\
\textbf{Ours}, frozen reader & \underline{59.7} & \underline{66.4} & \underline{39.6} & 57.0 & \underline{66.6} & \underline{57.8} & \underline{57.4} \\
\textbf{Ours} & \textbf{65.3} & \textbf{69.0} & \textbf{46.8} & \textbf{74.0} & \textbf{70.5} & \textbf{65.1} & \textbf{65.1} \\
\midrule
\multicolumn{8}{l}{\emph{196 tokens}} \\
Single view & 52.0 & 66.7 & 37.9 & \underline{77.8} & 68.3 & 60.5 & 62.7 \\
ViCrop~\cite{zhang2025vicrop} & 52.3 & 66.2 & 42.2 & 71.8 & 69.4 & 60.4 & 62.4 \\
AttWarp~\cite{dalal2026attwarp} & 53.7 & 65.8 & 42.9 & 76.8 & 72.6 & 62.3 & 64.5 \\
Ground-then-crop & 68.7 & 72.6 & 46.5 & 70.8 & 71.8 & 66.1 & 65.4 \\
GapSight$^\dagger$~\cite{zhu2026gapsight} & 63.0 & 73.1 & 46.3 & 75.8 & 67.7 & 65.2 & 65.7 \\
\textbf{Ours}, frozen reader & \underline{82.0} & \underline{80.8} & \underline{54.3} & 69.4 & \underline{74.4} & \underline{72.2} & \underline{69.7} \\
\textbf{Ours} & \textbf{83.7} & \textbf{83.0} & \textbf{59.6} & \textbf{79.6} & \textbf{75.9} & \textbf{76.3} & \textbf{74.5} \\
\midrule
\multicolumn{8}{l}{\emph{392 tokens}} \\
Single view & 77.7 & 81.8 & 53.6 & \underline{81.0} & 77.0 & 74.2 & 73.3 \\
ViCrop~\cite{zhang2025vicrop} & 76.0 & 78.7 & 55.8 & 79.2 & 73.5 & 72.6 & 71.8 \\
AttWarp~\cite{dalal2026attwarp} & 81.7 & 82.3 & 58.0 & 80.8 & \underline{78.1} & 76.2 & 74.8 \\
Ground-then-crop & 85.3 & 83.7 & 58.9 & 78.6 & 74.3 & 76.2 & 73.9 \\
GapSight$^\dagger$~\cite{zhu2026gapsight} & 79.7 & 84.2 & 56.7 & 79.0 & 76.2 & 75.2 & 74.0 \\
\textbf{Ours}, frozen reader & \underline{86.0} & \underline{86.1} & \underline{64.8} & 76.8 & 75.0 & \underline{77.8} & \underline{75.7} \\
\textbf{Ours} & \textbf{88.3} & \textbf{87.8} & \textbf{70.6} & \textbf{81.6} & \textbf{79.2} & \textbf{81.5} & \textbf{79.8} \\
\midrule
\multicolumn{8}{l}{\emph{784 tokens}} \\
Single view & 90.3 & \textbf{91.3} & 62.6 & \underline{82.4} & 79.8 & \underline{81.3} & 79.0 \\
ViCrop~\cite{zhang2025vicrop} & 87.0 & 87.0 & 63.9 & 80.0 & 76.6 & 78.9 & 76.9 \\
AttWarp~\cite{dalal2026attwarp} & 88.7 & 90.3 & 65.6 & 80.6 & \textbf{80.9} & 81.2 & \underline{79.3} \\
Ground-then-crop & 88.0 & 88.1 & 65.5 & 76.6 & 76.3 & 78.9 & 76.7 \\
GapSight$^\dagger$~\cite{zhu2026gapsight} & \underline{90.7} & \underline{91.0} & 64.9 & 79.2 & 78.6 & 80.9 & 78.4 \\
\textbf{Ours}, frozen reader & \underline{90.7} & 89.2 & \underline{69.7} & 78.2 & 76.8 & 80.9 & 78.5 \\
\textbf{Ours} & \textbf{92.3} & 89.9 & \textbf{72.6} & \textbf{82.6} & \underline{80.0} & \textbf{83.5} & \textbf{81.3} \\
\bottomrule
\end{tabular}
\end{table}

\begin{table}[t]
\centering
\footnotesize
\setlength{\tabcolsep}{1.9pt}
\caption{\textbf{Full results of Qwen3-VL-4B} at $\B=98$ to $784$ visual tokens per view (accuracy, \%, official metrics). Avg5: five-benchmark average; Avg4: without V-CoT-Doc. Ours, frozen reader: stage~1 only; $^\dagger$our re-implementation (Sec.~\ref{sec:supp_gapsight}). Q-Zoom: released weights. 784 tokens is the boundary regime with a small information gap. Bold: best; underlined: second best per budget and column. --: not run at this budget. $^\ast$Q-Zoom's released training data contain the V-CoT-Doc questions we evaluate on.}
\label{tab:full_4b}
\begin{tabular}{l ccccc cc}
\toprule
Method & V-CoT & Doc & Info & Chart & Text & Avg5 & Avg4 \\
\midrule
\multicolumn{8}{l}{\emph{98 tokens}} \\
Single view & 19.7 & 49.2 & 32.4 & 63.2 & 59.0 & 44.7 & 50.9 \\
ViCrop~\cite{zhang2025vicrop} & 31.0 & 51.5 & 40.2 & 64.2 & 67.1 & 50.8 & 55.7 \\
AttWarp~\cite{dalal2026attwarp} & 22.3 & 40.1 & 35.2 & 61.2 & 67.7 & 45.3 & 51.1 \\
Ground-then-crop & 44.3 & 59.0 & 47.5 & 66.8 & 70.0 & 57.5 & 60.8 \\
GapSight$^\dagger$~\cite{zhu2026gapsight} & 48.7 & 61.7 & 42.3 & 62.8 & 65.1 & 56.1 & 58.0 \\
Q-Zoom~\cite{shi2026qzoom} & 67.7$^\ast$ & 69.9 & 45.3 & \underline{72.0} & \underline{73.8} & 65.7$^\ast$ & \underline{65.2} \\
\textbf{Ours}, frozen reader & \underline{73.0} & \underline{72.3} & \underline{49.9} & 64.4 & 72.5 & \underline{66.4} & 64.8 \\
\textbf{Ours} & \textbf{76.0} & \textbf{75.5} & \textbf{54.8} & \textbf{76.6} & \textbf{77.2} & \textbf{72.0} & \textbf{71.0} \\
\midrule
\multicolumn{8}{l}{\emph{196 tokens}} \\
Single view & 55.7 & 70.2 & 45.6 & 82.4 & 71.1 & 65.0 & 67.3 \\
ViCrop~\cite{zhang2025vicrop} & 67.3 & 71.2 & 50.9 & 78.6 & 73.2 & 68.2 & 68.5 \\
AttWarp~\cite{dalal2026attwarp} & 59.7 & 71.7 & 49.1 & 79.0 & 75.6 & 67.0 & 68.9 \\
Ground-then-crop & 65.7 & 78.4 & 56.8 & 79.8 & 73.9 & 70.9 & 72.2 \\
GapSight$^\dagger$~\cite{zhu2026gapsight} & 73.0 & 77.7 & 53.3 & 78.4 & 72.8 & 71.0 & 70.5 \\
Q-Zoom~\cite{shi2026qzoom} & 81.0$^\ast$ & 82.5 & 64.2 & \textbf{84.8} & \underline{80.4} & \underline{78.6$^\ast$} & \underline{78.0} \\
\textbf{Ours}, frozen reader & \underline{85.0} & \underline{85.4} & \underline{64.7} & 78.6 & 76.0 & 77.9 & 76.2 \\
\textbf{Ours} & \textbf{87.3} & \textbf{87.4} & \textbf{67.4} & \underline{84.6} & \textbf{81.2} & \textbf{81.6} & \textbf{80.1} \\
\midrule
\multicolumn{8}{l}{\emph{392 tokens}} \\
Single view & 82.3 & 85.8 & 61.0 & 85.8 & 79.0 & 78.8 & 77.9 \\
ViCrop~\cite{zhang2025vicrop} & 84.7 & 85.2 & 65.8 & 83.8 & 76.7 & 79.2 & 77.9 \\
AttWarp~\cite{dalal2026attwarp} & 87.7 & 88.2 & 63.9 & 85.6 & \underline{82.4} & 81.6 & 80.0 \\
Ground-then-crop & 85.7 & 87.7 & 66.8 & 83.2 & 77.9 & 80.3 & 78.9 \\
GapSight$^\dagger$~\cite{zhu2026gapsight} & 83.3 & 87.4 & 64.2 & 80.8 & 75.3 & 78.2 & 77.0 \\
Q-Zoom~\cite{shi2026qzoom} & \textbf{91.3$^\ast$} & \underline{88.6} & \underline{73.8} & \textbf{87.2} & \textbf{82.5} & \underline{84.7$^\ast$} & \underline{83.0} \\
\textbf{Ours}, frozen reader & 89.7 & 88.1 & \underline{73.8} & 81.8 & 77.7 & 82.2 & 80.3 \\
\textbf{Ours} & \underline{90.7} & \textbf{89.8} & \textbf{75.3} & \underline{87.0} & \textbf{82.5} & \textbf{85.1} & \textbf{83.7} \\
\midrule
\multicolumn{8}{l}{\emph{784 tokens}} \\
Single view & 89.3 & \underline{93.9} & 70.8 & \underline{86.2} & \underline{82.0} & 84.5 & 83.2 \\
ViCrop~\cite{zhang2025vicrop} & 89.3 & 92.7 & 72.3 & 84.6 & 77.6 & 83.3 & 81.8 \\
AttWarp~\cite{dalal2026attwarp} & \underline{92.3} & \textbf{94.4} & 73.9 & 84.8 & 81.6 & \underline{85.4} & \underline{83.7} \\
Ground-then-crop & 90.0 & 93.3 & 75.2 & 84.0 & 78.4 & 84.2 & 82.7 \\
GapSight$^\dagger$~\cite{zhu2026gapsight} & 88.3 & 90.2 & 72.8 & 83.2 & 77.4 & 82.4 & 80.9 \\
Q-Zoom~\cite{shi2026qzoom} & -- & -- & -- & -- & -- & -- & -- \\
\textbf{Ours}, frozen reader & 90.7 & 91.1 & \underline{79.3} & 83.8 & 78.0 & 84.6 & 83.0 \\
\textbf{Ours} & \textbf{92.7} & 92.5 & \textbf{80.9} & \textbf{87.6} & \textbf{84.2} & \textbf{87.6} & \textbf{86.3} \\
\bottomrule
\end{tabular}
\end{table}

\begin{table}[t]
\centering
\footnotesize
\setlength{\tabcolsep}{1.9pt}
\caption{\textbf{Full results of Qwen3-VL-8B} at $\B=98$ to $784$ visual tokens per view (accuracy, \%, official metrics). Avg5: five-benchmark average; Avg4: without V-CoT-Doc. Ours, frozen reader: stage~1 only; $^\dagger$our re-implementation (Sec.~\ref{sec:supp_gapsight}). 784 tokens is the boundary regime with a small information gap. Bold: best; underlined: second best per budget and column.}
\label{tab:full_8b}
\begin{tabular}{l ccccc cc}
\toprule
Method & V-CoT & Doc & Info & Chart & Text & Avg5 & Avg4 \\
\midrule
\multicolumn{8}{l}{\emph{98 tokens}} \\
Single view & 24.0 & 50.4 & 36.9 & 67.8 & 60.3 & 47.9 & 53.9 \\
ViCrop~\cite{zhang2025vicrop} & 31.7 & 54.5 & 38.4 & 72.8 & 70.2 & 53.5 & 59.0 \\
AttWarp~\cite{dalal2026attwarp} & 28.0 & 40.4 & 38.4 & 63.2 & 65.5 & 47.1 & 51.9 \\
Ground-then-crop & 46.0 & 61.5 & 52.4 & 72.8 & \underline{73.6} & 61.3 & 65.1 \\
GapSight$^\dagger$~\cite{zhu2026gapsight} & 56.3 & 66.9 & 42.1 & 73.8 & 68.4 & 61.5 & 62.8 \\
\textbf{Ours}, frozen reader & \underline{75.3} & \underline{75.5} & \underline{55.0} & \underline{74.4} & 72.4 & \underline{70.5} & \underline{69.3} \\
\textbf{Ours} & \textbf{77.3} & \textbf{77.5} & \textbf{59.1} & \textbf{79.4} & \textbf{74.0} & \textbf{73.5} & \textbf{72.5} \\
\midrule
\multicolumn{8}{l}{\emph{196 tokens}} \\
Single view & 67.0 & 78.1 & 50.2 & 83.2 & 73.6 & 70.4 & 71.3 \\
ViCrop~\cite{zhang2025vicrop} & 69.0 & 77.0 & 51.6 & \underline{84.2} & 76.1 & 71.6 & 72.2 \\
AttWarp~\cite{dalal2026attwarp} & 69.0 & 75.5 & 53.8 & 81.0 & \underline{78.3} & 71.5 & 72.1 \\
Ground-then-crop & 76.0 & 82.2 & 62.9 & 84.0 & 77.4 & 76.5 & 76.6 \\
GapSight$^\dagger$~\cite{zhu2026gapsight} & 79.3 & 81.3 & 56.0 & 83.2 & 75.4 & 75.0 & 74.0 \\
\textbf{Ours}, frozen reader & \underline{87.0} & \underline{87.8} & \textbf{72.5} & 80.6 & 77.8 & \underline{81.1} & \underline{79.7} \\
\textbf{Ours} & \textbf{88.7} & \textbf{89.2} & \underline{72.3} & \textbf{87.6} & \textbf{80.2} & \textbf{83.6} & \textbf{82.3} \\
\midrule
\multicolumn{8}{l}{\emph{392 tokens}} \\
Single view & 87.7 & 89.3 & 64.5 & 87.0 & 80.2 & 81.8 & 80.3 \\
ViCrop~\cite{zhang2025vicrop} & 87.3 & 88.9 & 66.8 & 85.8 & 77.5 & 81.3 & 79.7 \\
AttWarp~\cite{dalal2026attwarp} & 90.0 & \underline{90.5} & 69.7 & \textbf{88.0} & 81.4 & 83.9 & 82.4 \\
Ground-then-crop & 89.3 & 89.7 & 74.7 & 85.2 & 80.0 & 83.8 & 82.4 \\
GapSight$^\dagger$~\cite{zhu2026gapsight} & 89.3 & \textbf{90.8} & 67.6 & 85.6 & 80.0 & 82.7 & 81.0 \\
\textbf{Ours}, frozen reader & \textbf{92.3} & 89.7 & \underline{79.9} & 84.4 & \underline{81.5} & \underline{85.6} & \underline{83.9} \\
\textbf{Ours} & \underline{92.0} & \textbf{90.8} & \textbf{82.9} & \underline{87.8} & \textbf{81.7} & \textbf{87.0} & \textbf{85.8} \\
\midrule
\multicolumn{8}{l}{\emph{784 tokens}} \\
Single view & \underline{93.0} & \underline{94.1} & 76.5 & \underline{87.0} & \underline{83.7} & 86.9 & 85.3 \\
ViCrop~\cite{zhang2025vicrop} & 92.0 & 92.4 & 75.6 & 85.8 & 81.3 & 85.4 & 83.8 \\
AttWarp~\cite{dalal2026attwarp} & \textbf{93.3} & \textbf{94.4} & 76.6 & \underline{87.0} & \textbf{83.9} & \underline{87.1} & 85.5 \\
Ground-then-crop & 92.3 & 93.6 & 80.0 & 85.4 & 80.8 & 86.4 & 84.9 \\
GapSight$^\dagger$~\cite{zhu2026gapsight} & 92.7 & 92.8 & 78.3 & 86.0 & 81.3 & 86.2 & 84.6 \\
\textbf{Ours}, frozen reader & 92.0 & 93.3 & \underline{82.3} & 85.6 & 82.0 & 87.0 & \underline{85.8} \\
\textbf{Ours} & 92.3 & 93.1 & \textbf{83.9} & \textbf{87.8} & 83.4 & \textbf{88.1} & \textbf{87.0} \\
\bottomrule
\end{tabular}
\end{table}

\noindent\textbf{Full evaluation sets.}
The main experiments use fixed random subsets of 500 questions for the public benchmarks.
Tab.~\ref{tab:fullset} repeats the comparison on the complete validation or test sets for Qwen3-VL-4B at 196 tokens: \ours again achieves the best average ($80.6$ \vs $78.1$ for Q-Zoom and $70.9$ for GapSight) and the best accuracy on four of the five benchmarks; its average lead is $2.5$ points over Q-Zoom, against $3.0$ on the subsets, and $9.7$ over GapSight, against $10.6$.
At 98 and 392 tokens, the full-set averages of \ours and Q-Zoom are $70.9$ \vs $65.7$ and $84.8$ \vs $84.4$.
\begin{table}[t]
\centering
\small
\setlength{\tabcolsep}{2.1pt}
\caption{\textbf{Full evaluation sets} (Qwen3-VL-4B, $\B=196$ visual tokens per view; accuracy, \%). Number of questions: V-CoT 300, Doc 5,349, Info 2,801, Chart 2,500, Text 5,000. V-CoT-Doc is evaluated in full in all experiments. Bold: best; underlined: second best per column. $^\dagger$Our re-implementation. $^\ast$Q-Zoom's released training data contain the V-CoT-Doc questions.}
\label{tab:fullset}
\begin{tabular}{l ccccc c}
\toprule
Method & V-CoT & Doc & Info & Chart & Text & Avg. \\
\midrule
Single view & 55.7 & 70.6 & 46.0 & 79.3 & 72.8 & 64.8 \\
ViCrop~\cite{zhang2025vicrop} & 67.3 & 72.9 & 50.5 & 77.4 & 74.7 & 68.6 \\
AttWarp~\cite{dalal2026attwarp} & 59.7 & 71.3 & 50.4 & 77.9 & 76.9 & 67.2 \\
Ground-then-crop & 65.7 & 76.8 & 57.6 & 77.4 & 74.5 & 70.4 \\
GapSight$^\dagger$~\cite{zhu2026gapsight} & 73.0 & 78.5 & 54.0 & 76.7 & 72.1 & 70.9 \\
Q-Zoom~\cite{shi2026qzoom} & 81.0$^\ast$ & 84.9 & 63.1 & \underline{81.4} & \textbf{80.3} & \underline{78.1$^\ast$} \\
\textbf{Ours}, frozen reader & \underline{85.0} & \underline{86.1} & \underline{63.9} & 76.6 & 75.5 & 77.4 \\
\textbf{Ours} & \textbf{87.3} & \textbf{87.4} & \textbf{67.3} & \textbf{82.2} & \underline{79.0} & \textbf{80.6} \\
\bottomrule
\end{tabular}
\end{table}

\noindent\textbf{Out-of-domain natural images.}
Tab.~\ref{tab:ood} evaluates V$^\ast$ and HR-Bench, natural photographs with small objects, far from the text-rich training data of the pointer and the reader.
\ours beats the single view on all 27 cells, by $2.1$ to $14.1$ points.
The reader matters more here than in domain.
With the frozen base reader, the window beats the single view on 16 of the 27 cells and falls below it on 10, including V$^\ast$ at every budget for 2B and 8B, whereas stage~2 improves all 27 cells, by $1.9$ to $28.8$ points, although it sees only text-rich images.
The gain of \ours comes from the zoomed view: the single-view accuracy of the stage-2 reader stays within $-1.1$ to $+2.9$ points of the base model's, and the window improves on it in all 27 cells, by $1.0$ to $15.2$ points.
Ground-then-crop is stronger than \ours on HR-Bench for the 4B model and comparable or weaker elsewhere (\ours is better on 16 of the 27 cells): on natural images the model can name the object it is asked about, and the named box, enlarged, contains it.
\begin{table*}[t]
\centering
\small
\setlength{\tabcolsep}{3.2pt}
\caption{\textbf{Out-of-domain natural images} (accuracy, \%): V$^\ast$~\cite{wu2024vstar} and HR-Bench~\cite{wang2024hrbench} 4K / 8K, under the same per-view budgets ($\B$ visual tokens per view). Ours, frozen reader: stage~1 only, the base model answers from the window of the \look token. Bold: best per model and column.}
\label{tab:ood}
\begin{tabular}{l l ccc ccc ccc}
\toprule
& & \multicolumn{3}{c}{$\B=98$} & \multicolumn{3}{c}{$\B=196$} & \multicolumn{3}{c}{$\B=392$} \\ \cmidrule(lr){3-5} \cmidrule(lr){6-8} \cmidrule(lr){9-11}
& Method & V$^\ast$ & HR-4K & HR-8K & V$^\ast$ & HR-4K & HR-8K & V$^\ast$ & HR-4K & HR-8K \\
\midrule
\multirow{4}{*}{\rotatebox[origin=c]{90}{2B}} & Single view & 39.3 & 44.9 & 38.8 & 49.2 & 51.2 & 40.9 & 57.1 & 54.1 & 46.5 \\
& Ground-then-crop & 35.6 & 46.6 & \textbf{41.6} & 40.8 & 55.6 & 46.2 & 48.2 & 59.9 & \textbf{55.9} \\
& \textbf{Ours}, frozen reader & 26.7 & 42.8 & 33.2 & 33.5 & 49.2 & 41.1 & 38.7 & 55.4 & 46.5 \\
& \textbf{Ours} & \textbf{53.4} & \textbf{48.6} & \textbf{41.6} & \textbf{55.5} & \textbf{57.1} & \textbf{48.6} & \textbf{67.5} & \textbf{62.9} & 54.4 \\
\midrule
\multirow{4}{*}{\rotatebox[origin=c]{90}{4B}} & Single view & 45.0 & 47.1 & 44.1 & 51.8 & 52.6 & 45.1 & 59.7 & 56.9 & 51.2 \\
& Ground-then-crop & 50.3 & \textbf{54.8} & \textbf{49.8} & \textbf{60.2} & \textbf{62.7} & \textbf{54.1} & 68.1 & \textbf{68.1} & \textbf{59.4} \\
& \textbf{Ours}, frozen reader & 49.7 & 50.4 & 47.1 & 52.9 & 56.0 & 47.8 & 66.5 & 59.5 & 54.0 \\
& \textbf{Ours} & \textbf{54.5} & 52.8 & 49.0 & 59.7 & 60.1 & 53.4 & \textbf{69.1} & 64.0 & 58.2 \\
\midrule
\multirow{4}{*}{\rotatebox[origin=c]{90}{8B}} & Single view & 43.5 & 47.8 & 44.9 & 56.0 & 54.9 & 48.1 & 64.9 & 59.8 & 53.4 \\
& Ground-then-crop & 46.1 & 55.4 & 47.8 & 55.5 & \textbf{64.6} & \textbf{54.0} & 63.9 & 66.2 & 61.0 \\
& \textbf{Ours}, frozen reader & 35.1 & 52.4 & 43.9 & 45.0 & 61.6 & 48.4 & 50.3 & 65.4 & 55.8 \\
& \textbf{Ours} & \textbf{53.4} & \textbf{57.6} & \textbf{51.2} & \textbf{60.7} & 64.4 & 53.5 & \textbf{67.0} & \textbf{69.9} & \textbf{61.2} \\
\bottomrule
\end{tabular}
\end{table*}

\noindent\textbf{Other architectures.}
We train \ours on InternVL3.5-2B and 4B~\cite{wang2025internvl35} and on Qwen2.5-VL-7B~\cite{bai2025qwen25vl} with the data, steps, hyperparameters, views and window rule used for Qwen3-VL, and run the training-free baselines on each base model as in Sec.~\ref{sec:supp_baselines}, with Ground-then-crop in each model's native grounding format; SD-RPN~\cite{shi2026sdrpn} is run from its released Qwen2.5-VL-7B weights.
The pointer is taken from layers 13--16 for all three models (of 28 layers for InternVL3.5-2B and Qwen2.5-VL-7B, of 36 for InternVL3.5-4B).
For InternVL3.5-4B, the band 17--20 of Qwen3-VL-4B gives a \look token whose support window covers $0.54$ of the evidence at 98 tokens on the further held-out questions of Sec.~\ref{sec:supp_impl}; comparing the bands 13--16, 17--20 and 21--24 on these questions only, we kept 13--16 ($0.79$, against $0.79$ and $0.81$ for InternVL3.5-2B and Qwen2.5-VL-7B).
In Tab.~\ref{tab:xarch}, the frozen reader already leads every training-free baseline on all three models at every budget, by $14.5$--$16.1$ points at 98 tokens, and the trained reader widens it to $15.2$--$18.9$ points; on Qwen2.5-VL-7B, \ours leads SD-RPN by $13.5$ / $3.0$ / $1.1$ points at 98 / 196 / 392 tokens.
On V$^\ast$ and HR-Bench, \ours improves on the single view in all 27 model--benchmark--budget settings and on Ground-then-crop in 24.
\begin{table}[t]
\centering
\small
\setlength{\tabcolsep}{3.2pt}
\caption{\textbf{Other architectures} (accuracy, \%): average over the five benchmarks (Avg5) and over V$^\ast$, HR-Bench 4K and 8K (OOD), at the per-view budgets of Qwen3-VL ($\B\times32^2$ pixels). Qwen2.5-VL (patch 14) encodes a view into about $1.3\B$ tokens, InternVL3.5 into $448^2$ tiles of 256 tokens (one tile at $\B=98$ and $196$, two and a thumbnail at $392$); all methods of a model use the same views. SD-RPN: released weights for Qwen2.5-VL-7B. --: not run. Bold: best per model and column.}
\label{tab:xarch}
\begin{tabular}{l ccc ccc}
\toprule
& \multicolumn{3}{c}{Avg5} & \multicolumn{3}{c}{OOD} \\ \cmidrule(lr){2-4} \cmidrule(lr){5-7}
Method & 98 & 196 & 392 & 98 & 196 & 392 \\
\multicolumn{7}{l}{\emph{InternVL3.5-2B}} \\
Single view & 45.1 & 48.3 & 68.5 & 45.7 & 46.4 & 50.6 \\
ViCrop & 50.6 & 53.7 & 72.7 & -- & -- & -- \\
AttWarp & 46.5 & 57.0 & 73.1 & -- & -- & -- \\
Ground-then-crop & 50.4 & 54.0 & 72.3 & 52.5 & 51.9 & 57.3 \\
\textbf{Ours}, frozen reader & 66.0 & 68.1 & 77.6 & -- & -- & -- \\
\textbf{Ours} & \textbf{68.4} & \textbf{69.5} & \textbf{79.7} & \textbf{52.7} & \textbf{54.3} & \textbf{59.5} \\
\midrule
\multicolumn{7}{l}{\emph{InternVL3.5-4B}} \\
Single view & 46.1 & 52.8 & 71.5 & 46.3 & 45.9 & 53.6 \\
ViCrop & 53.7 & 58.4 & 74.8 & -- & -- & -- \\
AttWarp & 48.7 & 60.6 & 75.6 & -- & -- & -- \\
Ground-then-crop & 54.1 & 60.2 & 77.3 & 49.7 & 50.9 & 58.9 \\
\textbf{Ours}, frozen reader & 68.6 & 71.0 & 80.2 & -- & -- & -- \\
\textbf{Ours} & \textbf{69.3} & \textbf{72.2} & \textbf{80.8} & \textbf{52.2} & \textbf{53.4} & \textbf{61.6} \\
\midrule
\multicolumn{7}{l}{\emph{Qwen2.5-VL-7B}} \\
Single view & 51.1 & 72.4 & 83.1 & 46.3 & 51.2 & 56.0 \\
ViCrop & 58.9 & 76.1 & 84.3 & -- & -- & -- \\
AttWarp & 48.3 & 72.1 & 84.4 & -- & -- & -- \\
Ground-then-crop & 57.1 & 74.3 & 83.2 & 47.8 & 51.5 & 58.4 \\
SD-RPN & 64.3 & 81.5 & 87.0 & -- & -- & -- \\
\textbf{Ours}, frozen reader & 75.0 & 82.3 & 86.3 & -- & -- & -- \\
\textbf{Ours} & \textbf{77.8} & \textbf{84.5} & \textbf{88.1} & \textbf{53.9} & \textbf{58.8} & \textbf{63.6} \\
\bottomrule
\end{tabular}
\end{table}

\noindent\textbf{Cost.}
Tab.~\ref{tab:cost} counts the passes, the generation before the answer, the visual tokens and the latency of every method on Qwen3-VL-4B at 196 tokens.
\ours needs two passes and generates nothing before the answer, so its latency of $0.25$\,s is below that of ViCrop, which needs three passes ($0.27$\,s), and of AttWarp ($0.44$\,s) and Ground-then-crop ($0.85$\,s), which first generate an answer or a box.
GapSight$^\dagger$ and Q-Zoom take $0.22$ and $0.21$\,s, while \ours is $10.6$ and $3.0$ points more accurate.
\begin{table}[t]
\centering
\small
\setlength{\tabcolsep}{2.6pt}
\caption{\textbf{Cost} on Qwen3-VL-4B at $\B=196$ visual tokens per view, in the accounting of AttWarp~\cite{dalal2026attwarp}: MLLM passes, generation before the final answer, visual tokens processed and latency per sample. Tokens and Acc.\ (accuracy) are five-benchmark averages; latency is the mean over 100 DocVQA questions (batch size 1, greedy decoding, one H200 GPU). $^\dagger$Our re-implementation. Q-Zoom reuses the lower layers of its first pass and zooms on 70\% of the questions of the five benchmarks; it is timed with its official code.}
\label{tab:cost}
\begin{tabular}{l ccccc}
\toprule
Method & Passes & Gen.\ first & Tokens & Lat.\ (s) & Acc. \\
\midrule
Single view & 1 & \xmark & 194 & 0.17 & 65.0 \\
ViCrop~\cite{zhang2025vicrop} & 3 & \xmark & 647 & 0.27 & 68.2 \\
AttWarp~\cite{dalal2026attwarp} & 2 & \cmark & 388 & 0.44 & 67.0 \\
Ground-then-crop & 2 & \cmark & 579 & 0.85 & 70.9 \\
GapSight$^\dagger$~\cite{zhu2026gapsight} & 2 & \xmark & 561 & 0.22 & 71.0 \\
Q-Zoom~\cite{shi2026qzoom} & 1--2 & \xmark & 331 & 0.21 & 78.6 \\
\textbf{Ours} & 2 & \xmark & 585 & 0.25 & 81.6 \\
\bottomrule
\end{tabular}
\end{table}

\noindent\textbf{Which pointer and which window?}
Tab.~\ref{tab:supp_pointer} extends Tab.~\ref{tab:ablation}c to all scales, with the frozen base reader.
With the same adaptive window, the \look token beats the pre-answer attention of the same layers at every scale, by $7.4$--$11.5$ points at 98 tokens, $3.1$--$5.8$ at 196 and $0.5$--$2.5$ at 392.
Taking the support window of the same pointer instead of its quantile window adds a further $5.5$--$8.6$, $3.0$--$3.5$ and $0.8$--$1.8$ points.
At its best layers, which have to be located with evidence boxes, the pre-answer attention is a stronger pointer than at the layers of the \look token; with the support window on both pointers, the \look token leads it at every scale and budget, by $4.4$--$5.6$ points at 98 tokens, $2.6$--$3.8$ at 196 and $0.4$--$2.7$ at 392.
\begin{table}[t]
\centering
\small
\setlength{\tabcolsep}{3.0pt}
\caption{\textbf{Pointer and window with the frozen base reader} (five-benchmark average accuracy, official metrics, at $\B=98$ / 196 / 392 visual tokens per view; two views). Pre-answer: attention of the answer-start position at the layers of the \look token or at its best layers, the four layers where it places the most mass on the evidence (17--20, 21--24 and 28--31 for 2B, 4B and 8B; located with evidence boxes). $\texttt{adapt}_{.8}$: the adaptive quantile window; $\texttt{supp}_2$: the support window of \ours; the last row of each block is Ours, frozen reader.}
\label{tab:supp_pointer}
\begin{tabular}{l ccc}
\toprule
Pointer, window & 98 & 196 & 392 \\
\midrule
2B, pre-answer, $\texttt{adapt}_{.8}$ & 40.8 & 63.2 & 73.5 \\
2B, pre-answer at its best layers, $\texttt{adapt}_{.8}$ & 47.5 & 66.9 & 76.7 \\
2B, pre-answer at its best layers, $\texttt{supp}_2$ & 53.4 & 68.3 & 77.4 \\
2B, \look token, $\texttt{adapt}_{.8}$ & 52.3 & 69.0 & 76.0 \\
2B, \look token, $\texttt{supp}_2$ & 57.8 & 72.2 & 77.8 \\
\midrule
4B, pre-answer, $\texttt{adapt}_{.8}$ & 50.4 & 70.1 & 79.6 \\
4B, pre-answer at its best layers, $\texttt{adapt}_{.8}$ & 54.1 & 72.1 & 80.4 \\
4B, pre-answer at its best layers, $\texttt{supp}_2$ & 61.6 & 75.3 & 80.4 \\
4B, \look token, $\texttt{adapt}_{.8}$ & 60.0 & 74.4 & 81.4 \\
4B, \look token, $\texttt{supp}_2$ & 66.4 & 77.9 & 82.2 \\
\midrule
8B, pre-answer, $\texttt{adapt}_{.8}$ & 54.5 & 75.0 & 83.8 \\
8B, pre-answer at its best layers, $\texttt{adapt}_{.8}$ & 61.2 & 77.5 & 84.5 \\
8B, pre-answer at its best layers, $\texttt{supp}_2$ & 64.9 & 77.3 & 82.9 \\
8B, \look token, $\texttt{adapt}_{.8}$ & 61.9 & 78.1 & 84.3 \\
8B, \look token, $\texttt{supp}_2$ & 70.5 & 81.1 & 85.6 \\
\bottomrule
\end{tabular}
\end{table}

\noindent\textbf{GapSight with the same reader.}
Tab.~\ref{tab:gs_reader} extends Tab.~\ref{tab:ablation}b to all scales and budgets.
With the base reader, \ours leads GapSight at every scale, by $9.0$--$10.3$ points at 98 tokens, $6.1$--$7.0$ at 196 and $2.6$--$4.0$ at 392.
When our stage-2 reader answers from the decisions and crops of GapSight, \ours leads by $9.4$--$11.9$, $6.8$--$8.3$ and $2.9$--$4.2$ points, so its gains come from where and how it looks rather than from the reader.
\begin{table}[t]
\centering
\small
\setlength{\tabcolsep}{4.0pt}
\caption{\textbf{GapSight$^\dagger$ with the same reader} (five-benchmark average accuracy, \%, at $\B=98$ to 784 visual tokens per view). The first two rows answer with the base model, the last two with our stage-2 reader; the third row keeps GapSight's decisions and crops, so each pair differs only in where and how to look. Bold: best per scale and budget. $^\dagger$Our re-implementation.}
\label{tab:gs_reader}
\begin{tabular}{@{}l cccc@{}}
\toprule
Method & 98 & 196 & 392 & 784 \\
\midrule
\multicolumn{5}{l}{\emph{Qwen3-VL-2B}} \\
GapSight$^\dagger$, base reader & 48.4 & 65.2 & 75.2 & 80.9 \\
\textbf{Ours}, frozen reader & 57.8 & 72.2 & 77.8 & 80.9 \\
Stage-2 reader, crops of GapSight$^\dagger$ & 54.3 & 68.0 & 77.3 & 80.3 \\
\textbf{Ours} & \textbf{65.1} & \textbf{76.3} & \textbf{81.5} & \textbf{83.5} \\
\midrule
\multicolumn{5}{l}{\emph{Qwen3-VL-4B}} \\
GapSight$^\dagger$, base reader & 56.1 & 71.0 & 78.2 & 82.4 \\
\textbf{Ours}, frozen reader & 66.4 & 77.9 & 82.2 & 84.6 \\
Stage-2 reader, crops of GapSight$^\dagger$ & 60.1 & 74.8 & 82.2 & 85.4 \\
\textbf{Ours} & \textbf{72.0} & \textbf{81.6} & \textbf{85.1} & \textbf{87.6} \\
\midrule
\multicolumn{5}{l}{\emph{Qwen3-VL-8B}} \\
GapSight$^\dagger$, base reader & 61.5 & 75.0 & 82.7 & 86.2 \\
\textbf{Ours}, frozen reader & 70.5 & 81.1 & 85.6 & 87.0 \\
Stage-2 reader, crops of GapSight$^\dagger$ & 64.1 & 76.5 & 83.9 & 86.9 \\
\textbf{Ours} & \textbf{73.5} & \textbf{83.6} & \textbf{87.0} & \textbf{88.1} \\
\bottomrule
\end{tabular}
\end{table}

\noindent\textbf{Length of stage~2.}
Tab.~\ref{tab:reader_steps} varies the number of GRPO steps of the reader.
The first 400 steps bring most of the gain over the frozen reader; going to 800 steps changes the average by $-0.6$ to $+1.7$ points (the gain is largest for 4B), and going to 1{,}200 steps by $-0.3$ to $+1.0$ points.
We use 800 steps at every scale.
\begin{table}[t]
\centering
\footnotesize
\setlength{\tabcolsep}{3pt}
\caption{\textbf{Length of stage~2} (five-benchmark average accuracy, \%, at $\B=98$ / 196 / 392 visual tokens per view). 0 steps is the frozen base reader (Ours, frozen reader); 800 steps is Ours. Bold: best per column.}
\label{tab:reader_steps}
\begin{tabular}{l ccc ccc ccc}
\toprule
& \multicolumn{3}{c}{2B} & \multicolumn{3}{c}{4B} & \multicolumn{3}{c}{8B} \\ \cmidrule(lr){2-4} \cmidrule(lr){5-7} \cmidrule(lr){8-10}
GRPO steps & 98 & 196 & 392 & 98 & 196 & 392 & 98 & 196 & 392 \\
\midrule
0 & 57.8 & 72.2 & 77.8 & 66.4 & 77.9 & 82.2 & 70.5 & 81.1 & 85.6 \\
400 & 64.8 & 76.9 & 81.4 & 70.3 & 80.0 & 83.8 & 72.8 & 82.8 & 86.8 \\
800 (\textbf{Ours}) & 65.1 & 76.3 & 81.5 & \textbf{72.0} & 81.6 & 85.1 & \textbf{73.5} & 83.6 & 87.0 \\
1{,}200 & \textbf{65.6} & \textbf{77.3} & \textbf{81.8} & 71.7 & \textbf{81.8} & \textbf{85.8} & \textbf{73.5} & \textbf{83.7} & \textbf{87.3} \\
\bottomrule
\end{tabular}
\end{table}

\noindent\textbf{The information gap (Sec.~\ref{sec:analysis}).}
Proposition~\ref{prop:gap} formalizes, in an idealized setting, why the benefit of zooming grows with the information gap.
\begin{proposition}[The information gap bounds what a zoomed view adds]
\label{prop:gap}
Let an image have $N$ native pixels and let every view be encoded with at most $P=\B\times32^2$ pixels, so that the information gap is $\rho=N/P$.
A view of a region with relative area $A\in(0,1]$ samples the region with at most $\min(1, P/(AN))$ of the native pixel density, since resampling above the native resolution adds no detail, and the budgeted view samples it with $\min(1,1/\rho)$.
The density gained on evidence inside the region is therefore
\begin{equation}
G(\rho,A)=\max\big(1,\ \min(\rho,\,1/A)\big).
\end{equation}
\end{proposition}
\noindent\emph{Proof.}
For $\rho\le1$ both densities equal $1$.
For $\rho>1$ the density of the budgeted view is $1/\rho$, and the zoomed density is $1/(\rho A)$ if $\rho A\ge1$ and $1$ otherwise, which gives the ratio $\min(\rho,1/A)$.

\noindent Three consequences match Fig.~\ref{fig:infogap}.
(i)~When $\rho\le1$ no view adds detail, and for small $\rho$ the gain is bounded by $\rho$ itself, while a crop removes context; zooming indeed costs points for $\rho<1.6$ (Fig.~\ref{fig:infogap}).
(ii)~For a fixed window the gain grows with $\rho$ until $\rho=1/A$.
(iii)~Shrinking the window below $A=1/\rho$ adds no detail but raises the risk of missing the evidence, so the useful window shrinks as $\rho$ grows: the best fixed view moves from gentle views at small gaps to strong magnification at large gaps.

\noindent\textbf{Is the zoomed view used?}
Tab.~\ref{tab:swap} gives the stage-2 reader the same square window at a random location, where the pre-answer attention or the \look token points, or centered on the ground-truth evidence.
Random windows help little ($+6.0$ / $+2.9$ / $+0.1$ points over the single view on the five-benchmark average of 4B at 98 / 196 / 392 tokens).
Windows placed by the \look token help substantially ($+22.6$ / $+13.7$ / $+4.6$), more than those placed by the pre-answer attention of the same model ($+3.2$ / $+1.7$ / $+0.3$ over it), and on V-CoT-Doc they come within $2$ points of the window centered on the ground-truth evidence: the reader relies on the content of the zoomed view.
The support window of \ours (last column) improves on the square window at the same pointer in every row and, at 98 and 196 tokens, on the ground-truth-centered square, because its size follows the evidence.
\begin{table}[t]
\centering
\footnotesize
\setlength{\tabcolsep}{2.4pt}
\caption{\textbf{Is the zoomed view used?} Accuracy of the stage-2 reader when the same square window ($0.5\min(W,H)$) is placed at random, by the pre-answer attention, by the \look token, or centered on the ground-truth evidence (V-CoT-Doc only), and with the support window of \ours ($\texttt{supp}_2$). $\B$: visual tokens per view. Pointed windows are far better than random ones, so the model relies on the content of the zoomed view. Bold: best per row.}
\label{tab:swap}
\begin{tabular}{ll cccccc}
\toprule
& $\B$ & Single & Random & Pre-ans. & \look & GT & $\texttt{supp}_2$ \\
\midrule
\multirow{3}{*}{\shortstack[l]{4B\\V-CoT-Doc}} & 98 & 19.7 & 34.3 & 63.3 & 69.0 & 71.0 & \textbf{76.0} \\
& 196 & 59.0 & 64.3 & 82.3 & 85.3 & 83.7 & \textbf{87.3} \\
& 392 & 82.0 & 83.7 & 90.0 & 89.3 & \textbf{91.0} & 90.7 \\
\midrule
\multirow{3}{*}{\shortstack[l]{4B\\5-avg.}} & 98 & 45.9 & 51.9 & 65.3 & 68.5 & -- & \textbf{72.0} \\
& 196 & 66.1 & 69.0 & 78.1 & 79.8 & -- & \textbf{81.6} \\
& 392 & 79.2 & 79.3 & 83.5 & 83.8 & -- & \textbf{85.1} \\
\midrule
\multirow{3}{*}{\shortstack[l]{2B\\V-CoT-Doc}} & 98 & 17.3 & 31.0 & 55.0 & 62.7 & 64.0 & \textbf{65.3} \\
& 196 & 54.0 & 59.7 & 76.0 & 80.7 & 81.0 & \textbf{83.7} \\
& 392 & 79.0 & 79.7 & 87.0 & 87.0 & \textbf{88.3} & \textbf{88.3} \\
\bottomrule
\end{tabular}
\end{table}

\noindent\textbf{What does zooming fix and what does it break?}
Tab.~\ref{tab:rescue} compares every answer with the single view of the base model.
\ours corrects the most single-view errors at every scale and budget ($42.6$--$54.9\%$) and breaks the fewest correct answers ($3.2$--$8.9\%$) in all cells but one: for 2B at 392 tokens GapSight, which zooms on fewer questions, breaks $4.4\%$ against $5.3\%$.
Against the single view of its own stage-2 reader the counts barely change (\eg $+372$ / $-55$ answers \vs $+392$ / $-49$ for 4B at 196 tokens), so the corrections come from the zoomed view rather than from the reader.
With the frozen base reader, the pointer corrects more errors than every baseline; stage~2 adds corrections and reduces the answers that the window breaks by a quarter to over a half.
\begin{table}[t]
\centering
\small
\setlength{\tabcolsep}{2.1pt}
\caption{\textbf{What zooming fixes and what it breaks}, relative to the single view of the base model (five benchmarks pooled, 2{,}300 questions per cell; column groups: $\B$ visual tokens per view). Rescue: share of the single-view errors that become correct; harm: share of the single-view correct answers that become wrong (an answer counts as correct when its score is at least $0.5$; strict match on V-CoT-Doc). Bold: highest rescue and lowest harm per scale and budget. $^\dagger$Our re-implementation.}
\label{tab:rescue}
\begin{tabular}{@{}l cc cc cc@{}}
\toprule
& \multicolumn{2}{c}{98} & \multicolumn{2}{c}{196} & \multicolumn{2}{c}{392} \\ \cmidrule(lr){2-3} \cmidrule(lr){4-5} \cmidrule(lr){6-7}
Method & Rescue & Harm & Rescue & Harm & Rescue & Harm \\
\midrule
\multicolumn{7}{l}{\emph{Qwen3-VL-2B}} \\
ViCrop~\cite{zhang2025vicrop} & 15.2 & 17.6 & 17.6 & 11.3 & 18.8 & 8.6 \\
AttWarp~\cite{dalal2026attwarp} & 17.1 & 22.4 & 28.3 & 13.4 & 27.7 & 6.9 \\
Ground-then-crop & 28.9 & 17.6 & 30.6 & 12.0 & 31.2 & 8.6 \\
GapSight$^\dagger$~\cite{zhu2026gapsight} & 24.4 & 17.2 & 22.1 & 7.3 & 16.0 & \textbf{4.4} \\
\textbf{Ours}, frozen reader & 39.4 & 15.9 & 42.2 & 10.2 & 35.5 & 8.2 \\
\textbf{Ours} & \textbf{47.5} & \textbf{8.9} & \textbf{48.8} & \textbf{7.0} & \textbf{42.9} & 5.3 \\
\midrule
\multicolumn{7}{l}{\emph{Qwen3-VL-4B}} \\
ViCrop~\cite{zhang2025vicrop} & 21.6 & 11.3 & 24.3 & 8.1 & 21.0 & 5.3 \\
AttWarp~\cite{dalal2026attwarp} & 20.5 & 20.6 & 28.6 & 10.6 & 31.2 & 4.8 \\
Ground-then-crop & 33.8 & 11.7 & 32.4 & 7.6 & 30.8 & 6.4 \\
GapSight$^\dagger$~\cite{zhu2026gapsight} & 29.1 & 10.8 & 33.3 & 8.6 & 24.2 & 7.4 \\
\textbf{Ours}, frozen reader & 47.1 & 11.4 & 48.1 & 7.0 & 37.8 & 6.4 \\
\textbf{Ours} & \textbf{54.2} & \textbf{6.8} & \textbf{52.4} & \textbf{3.2} & \textbf{42.6} & \textbf{3.7} \\
\midrule
\multicolumn{7}{l}{\emph{Qwen3-VL-8B}} \\
ViCrop~\cite{zhang2025vicrop} & 20.6 & 8.6 & 22.9 & 7.4 & 22.8 & 5.5 \\
AttWarp~\cite{dalal2026attwarp} & 20.6 & 21.3 & 32.0 & 10.6 & 31.8 & 4.0 \\
Ground-then-crop & 34.5 & 8.1 & 36.6 & 6.5 & 34.7 & 4.7 \\
GapSight$^\dagger$~\cite{zhu2026gapsight} & 33.4 & 8.7 & 31.4 & 6.7 & 23.8 & 3.8 \\
\textbf{Ours}, frozen reader & 49.2 & 8.0 & 51.4 & 6.8 & 42.9 & 4.7 \\
\textbf{Ours} & \textbf{52.6} & \textbf{5.3} & \textbf{54.9} & \textbf{4.5} & \textbf{46.2} & \textbf{3.5} \\
\bottomrule
\end{tabular}
\end{table}

\noindent\textbf{Where the zoomed window lands.}
Tab.~\ref{tab:coverage}a measures how much of the evidence lies inside the region that each method views at high resolution.
The box named by Ground-then-crop is enlarged by its own size on each side and still contains less of the evidence than the region selected by \ours at every budget; the gap narrows at 392 tokens, where the text becomes legible enough to be named.
At 98 tokens the crop of GapSight lies in between ($0.52$--$0.61$), and it covers less at 392 tokens ($0.19$--$0.50$; questions without a zoomed view count as 0).
The first crop of DeepEyes-7B contains the least evidence (Sec.~\ref{sec:supp_baselines}).
Tab.~\ref{tab:coverage}b compares, on V-CoT-Doc with Qwen3-VL-4B, the window that \ours and Q-Zoom actually zoom into: its median area, how often it covers at least half of the evidence (hit), and the accuracy on hits and misses.
At 98 tokens, Q-Zoom's proposer selects a window of a tenth of the image that hits on $69\%$ of the questions it zooms on, while the support window of \ours covers a fifth of the image and hits on $81\%$ of all questions.
At 196 and 392 tokens, \ours hits on $85\%$ and $83\%$ of the questions, against $73\%$ and $77\%$ of those that Q-Zoom zooms on.
On hits, \ours is as accurate as Q-Zoom at 98 and 392 tokens ($81.8$ \vs $81.3$ and $94.0$ \vs $94.0$) and more accurate at 196 ($92.1$ \vs $88.0$).
Q-Zoom declines to zoom on $7\%$, $12\%$ and $42\%$ of the questions at 98, 196 and 392 tokens; at 392 tokens the two methods are within one point on this benchmark.
\begin{table*}[t]
\centering
\footnotesize
\setlength{\tabcolsep}{4pt}
\caption{\textbf{Where the zoomed window lands} (V-CoT-Doc, 300 questions; $\B$: visual tokens per view). (a)~Coverage: fraction of the area of the evidence boxes that falls inside the region viewed at high resolution; questions without a zoomed view count as 0. Ground-then-crop uses the named box enlarged by its own size on each side; GapSight$^\dagger$ (our re-implementation), the crop of its router; DeepEyes, the first crop of its trajectory. Bold: best per model and budget. (b)~Window size and hit rate against Q-Zoom (Qwen3-VL-4B). Area: median area of the zoomed window as a fraction of the image; no zoom: share of the questions answered from the budgeted view alone; hit: the window covers at least half of the evidence; accuracy overall and on hits / misses. Q-Zoom's hit rate and miss accuracy are computed over the questions it zooms on; Ours zooms on every question.}
\label{tab:coverage}
\begin{tabular}[t]{l ccc}
\toprule
(a) Coverage & 98 & 196 & 392 \\
\midrule
\multicolumn{4}{l}{\emph{Qwen3-VL-2B}} \\
Ground-then-crop & 0.41 & 0.52 & 0.69 \\
GapSight$^\dagger$ & 0.61 & 0.33 & 0.19 \\
\textbf{Ours} & \textbf{0.73} & \textbf{0.83} & \textbf{0.83} \\
\midrule
\multicolumn{4}{l}{\emph{Qwen3-VL-4B}} \\
Ground-then-crop & 0.36 & 0.40 & 0.65 \\
GapSight$^\dagger$ & 0.54 & 0.59 & 0.50 \\
\textbf{Ours} & \textbf{0.78} & \textbf{0.82} & \textbf{0.82} \\
\midrule
\multicolumn{4}{l}{\emph{Qwen3-VL-8B}} \\
Ground-then-crop & 0.34 & 0.41 & 0.58 \\
GapSight$^\dagger$ & 0.52 & 0.50 & 0.28 \\
\textbf{Ours} & \textbf{0.78} & \textbf{0.82} & \textbf{0.82} \\
\midrule
\multicolumn{4}{l}{\emph{Qwen2.5-VL-7B}} \\
DeepEyes-7B, first crop & 0.15 & 0.09 & 0.21 \\
\bottomrule
\end{tabular}
\hspace{3em}
\setlength{\tabcolsep}{2.2pt}
\begin{tabular}[t]{l l ccccc}
\toprule
(b) $\B$ & Method & Area & No zoom & Hit & Acc. & Acc.\ hit / miss \\
\midrule
98 & Q-Zoom & 0.10 & 0.07 & 0.69 & 67.7 & 81.3 / 43.0 \\
98 & \textbf{Ours} & 0.20 & 0.00 & 0.81 & 76.0 & 81.8 / 51.7 \\
\midrule
196 & Q-Zoom & 0.06 & 0.12 & 0.73 & 81.0 & 88.0 / 62.0 \\
196 & \textbf{Ours} & 0.17 & 0.00 & 0.85 & 87.3 & 92.1 / 60.9 \\
\midrule
392 & Q-Zoom & 0.05 & 0.42 & 0.77 & 91.3 & 94.0 / 87.5 \\
392 & \textbf{Ours} & 0.14 & 0.00 & 0.83 & 90.7 & 94.0 / 74.5 \\
\bottomrule
\end{tabular}

\end{table*}

\section{Baselines}
\label{sec:supp_baselines}
This section details how each baseline of Sec.~\ref{sec:exp_setup} is run and checks the ones with public code against the numbers reported by their authors.

\noindent\textbf{ViCrop and AttWarp.}
We re-implement ViCrop and AttWarp from the official repositories (ViCrop: \texttt{mllms\_know}; AttWarp: official release), keeping all default hyperparameters.
ViCrop uses the relative attention at layer 22 of 36 (the official Qwen2.5-VL setting, rescaled to depth for the 28-layer model), square windows of $224\times\{1,1.2,\dots,2\}$ pixels selected by the official neighbor-difference criterion, and answers from the original image and the raw crop.
AttWarp uses the attention of all generated tokens at layer 16 of 28 (rescaled to depth), the official min--max normalization, z-score$\times$10 sigmoid and $3{\times}3$ smoothing, and the official inverse-CDF warp, and it answers from the warped image, whose size we set to that of the budgeted view.
Tab.~\ref{tab:repro} checks the baselines with public code in settings reported by their authors: ViCrop and AttWarp in that of AttWarp, and DeepEyes-7B on V$^\ast$, where we obtain $88.5$ against the reported $90.1$.
\begin{table*}[t]
\centering
\footnotesize
\setlength{\tabcolsep}{2.4pt}
\caption{\textbf{Reproduction of the baselines with public code}, each in a setting reported by its authors. Left: ViCrop and AttWarp, re-implemented line by line from the official code, in the setting of AttWarp~\cite{dalal2026attwarp} (Qwen2.5-VL-7B, $512^2$ images). They reproduce the reported improvements over the base model on DocVQA and also improve TextVQA, by smaller margins; absolute numbers differ because we evaluate a 1{,}000-question subset with our prompt. Right: DeepEyes-7B~\cite{zheng2026deepeyes} with the official checkpoint, prompt and tool loop on V*~\cite{wu2024vstar} at native resolution; the 25 of 191 answers that the official rules leave undecided are judged by Qwen2.5-VL-7B instead of the official Qwen2.5-72B.}
\label{tab:repro}
\begin{tabular}[t]{l l ccc}
\toprule
Benchmark & Source & Base & ViCrop & AttWarp \\
\midrule
\multirow{2}{*}{\shortstack[l]{DocVQA\\(EM)}} & reported & 77.3 & 82.5 \scriptsize{(+5.2)} & 84.1 \scriptsize{(+6.8)} \\
 & ours & 72.4 & 79.0 \scriptsize{(+6.6)} & 81.5 \scriptsize{(+9.1)} \\
\midrule
\multirow{2}{*}{\shortstack[l]{TextVQA\\(VQA, bin.)}} & reported & 81.0 & 83.8 \scriptsize{(+2.8)} & 84.7 \scriptsize{(+3.7)} \\
 & ours & 75.4 & 77.2 \scriptsize{(+1.8)} & 76.7 \scriptsize{(+1.3)} \\
\bottomrule
\end{tabular}
\hspace{3em}
\begin{tabular}[t]{l ccc}
\toprule
DeepEyes-7B, V* & Overall & Attribute & Spatial \\
\midrule
reported & 90.1 & 91.3 & 88.2 \\
ours & 88.5 & 92.2 & 82.9 \\
\bottomrule
\end{tabular}
\end{table*}

\noindent\textbf{Ground-then-crop.}
Ground-then-crop is our own construction (prompt in Sec.~\ref{sec:supp_impl}): the model names the evidence box, which is enlarged by its own size on each side and re-encoded within the budget as a second view.
Like GapSight, which has no public code, it has no reported numbers that we could reproduce; for GapSight we report the diagnostics of our re-implementation instead.

\noindent\textbf{GapSight.}
\label{sec:supp_gapsight}
GapSight~\cite{zhu2026gapsight} trains a router on the frozen model that takes the hidden states of the budgeted view and decides whether to add a crop and where, supervised by the loss gap of the answer, the negative log-likelihood given the view alone minus that given the view and a candidate crop.
No code is public, so we re-implement it from the paper and keep everything else as in our setting.
The router of each model is trained on 6{,}000 training images from our training sources, at 98, 196 and 392 tokens.
For each image and budget we score 24 candidate crops, 12 compact windows covering $2$--$15\%$ of the image with aspect ratios between $1/3$ and $3$, each also enlarged $1.8\times$.
The $40.4\%$ of the samples with the largest gap are labeled \emph{review} and the $23.5\%$ with the smallest \emph{keep}, the proportions reported by GapSight, and the rest are left out of the gate loss.
The router takes the hidden state of the last prompt position at layer $\ell_{\max}{+}1$ and at the last layer, together with the visual-token states of layer $\ell_{\max}{+}1$ pooled in 39 prior boxes, and predicts the gate, the utility, the prior box and a bounded residual, trained with the five losses of GapSight.
The gate threshold is set for each budget on 10\% held-out images to maximize accuracy (784 tokens uses that of 392), and the model answers from the budgeted view and the crop with the prompt of \ours.
On held-out images the gate reaches an AUROC of $0.85$ / $0.89$ / $0.88$ (2B / 4B / 8B), while the predicted prior box is the one that overlaps the best candidate crop most for $11$--$15\%$ of the samples labeled \emph{review} (chance: $2.6\%$).
Sec.~\ref{sec:supp_exp} compares GapSight and \ours with the same reader (Tab.~\ref{tab:gs_reader}).

\noindent\textbf{Q-Zoom.}
Q-Zoom~\cite{shi2026qzoom} trains a region proposer (SD-RPN~\cite{shi2026sdrpn}) by self-distillation from the attention of the model's own responses, a gate that decides whether the high-resolution region is needed, and fine-tunes the language model on hard samples.
We run the released Qwen3-VL-4B weights with the official code and change only the image protocol: the budgeted view is encoded within the per-view budget, and the region it proposes, when the gate fires, is re-encoded within the same budget as a second view, as for \ours.
Its released training files contain the 300 V-CoT-Doc questions we evaluate on (7 of them in its supervised fine-tuning set), which is why Tab.~\ref{tab:main} also reports averages without this benchmark.
Q-Zoom has no released weights for the 2B and 8B models.

\begin{table*}[t]
\centering
\small
\setlength{\tabcolsep}{3.6pt}
\caption{\textbf{DeepEyes-7B under per-view budgets} (accuracy, \%). Every image of a trajectory, the input view and each crop, is capped at the same number of pixels as in our experiments ($\B\times32^2$; with the $28$-pixel tokens of Qwen2.5-VL this allows up to $128$ / $256$ / $512$ tokens per image). Visual tokens are counted over all images of a trajectory. Its base model, Qwen2.5-VL-7B, is shown with a single view.}
\label{tab:deepeyes}
\begin{tabular}{l ccccc c cc}
\toprule
Method & V-CoT-Doc & DocVQA & InfoVQA & ChartQA & TextVQA & Avg. & Visual tokens & Tool calls \\
\midrule
\multicolumn{9}{l}{\emph{$\B=98$: $98\times32^2$ pixels per image}} \\
Qwen2.5-VL-7B, single view & 33.7 & 56.0 & 38.7 & 70.0 & 61.8 & 52.0 & 128 & 0.00 \\
DeepEyes-7B, official prompt & 22.7 & 38.1 & 25.8 & 33.0 & 29.6 & 29.8 & 199 & 1.68 \\
DeepEyes-7B + answer-format instruction & 30.7 & 47.8 & 32.0 & 51.4 & 49.1 & 42.2 & 192 & 1.58 \\
\midrule
\multicolumn{9}{l}{\emph{$\B=196$: $196\times32^2$ pixels per image}} \\
Qwen2.5-VL-7B, single view & 73.3 & 80.1 & 53.5 & 82.8 & 74.6 & 72.9 & 252 & 0.00 \\
DeepEyes-7B, official prompt & 53.3 & 53.1 & 37.4 & 39.8 & 36.3 & 44.0 & 332 & 1.66 \\
DeepEyes-7B + answer-format instruction & 65.3 & 65.8 & 42.8 & 60.4 & 54.3 & 57.7 & 327 & 1.60 \\
\midrule
\multicolumn{9}{l}{\emph{$\B=392$: $392\times32^2$ pixels per image}} \\
Qwen2.5-VL-7B, single view & 87.7 & 91.6 & 67.7 & 87.4 & 82.1 & 83.3 & 508 & 0.00 \\
DeepEyes-7B + answer-format instruction & 80.3 & 77.5 & 54.7 & 65.8 & 58.5 & 67.4 & 564 & 1.57 \\
\bottomrule
\end{tabular}
\end{table*}

\noindent\textbf{DeepEyes.}
DeepEyes-7B~\cite{zheng2026deepeyes} is an RL-trained agent that zooms by emitting box coordinates during multi-turn reasoning; its base model is Qwen2.5-VL-7B.
We run the official system prompt and tool loop with vLLM, cap every image of a trajectory (the input view and each crop) at the per-view pixel budget, and add the answer-format instruction of our prompts to the question, which raises its accuracy over the official prompt alone (Tab.~\ref{tab:deepeyes}).
Because its base model differs from ours, we compare it with its own base model given a single view rather than with the rows of Tab.~\ref{tab:main}.
Under these budgets DeepEyes trails its base model by $9.8$ / $15.2$ / $15.9$ points at $98$ / $196$ / $392\times32^2$ pixels per image while processing more visual tokens, and it answers with more than six words in 7--20\% of the questions, depending on the benchmark.
Its crops contain little of the evidence, a likely cause of the drop: on V-CoT-Doc the first crop contains only $15$ / $9$ / $21\%$ of the evidence area (Tab.~\ref{tab:coverage}a), and in the trajectories we inspected it often lands near the evidence but misses the line that holds the answer, which the budgeted view does not let the model read.

\section{Qualitative results}
\label{sec:supp_qual}
This section shows the windows that \ours selects and the answers of all methods on examples from the five benchmarks (Qwen3-VL-4B, 196 tokens), including typical failures.
\begin{table*}[!t]\captionsetup{type=figure}   %
  \centering
  \includegraphics[width=0.90\textwidth]{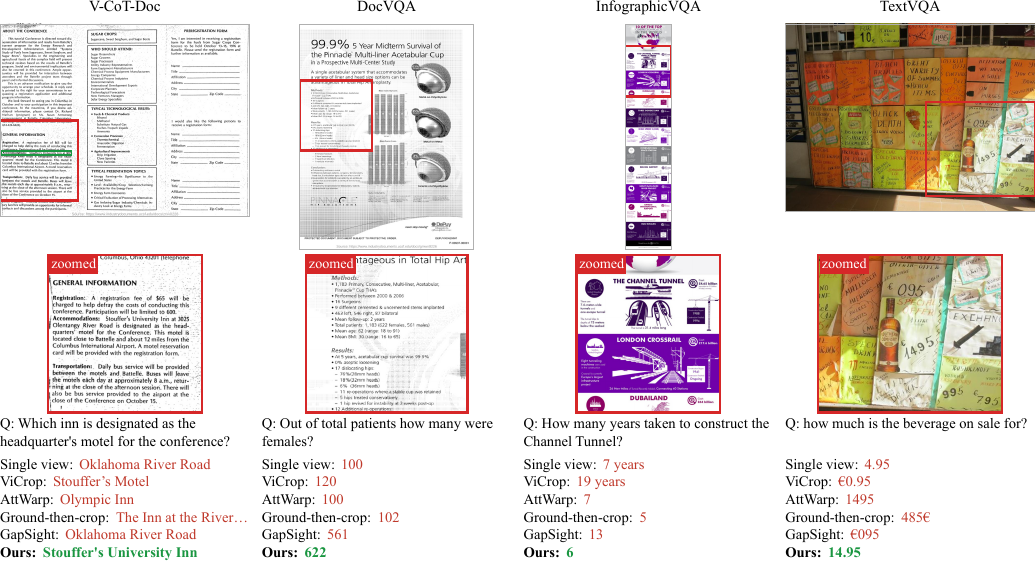}
  \caption{\textbf{Qualitative results} (Qwen3-VL-4B, 196 tokens). Top: input with the window selected by \ours (red). Middle: the zoomed view as the model reads it, re-encoded within the budget. Bottom: answers (green: correct, red: wrong). The selected views make names, numbers and prices legible that the single view misreads.}
  \label{fig:qual}
\end{table*}
\begin{table*}[!t]\captionsetup{type=figure}   %
  \centering
  \includegraphics[width=0.90\textwidth]{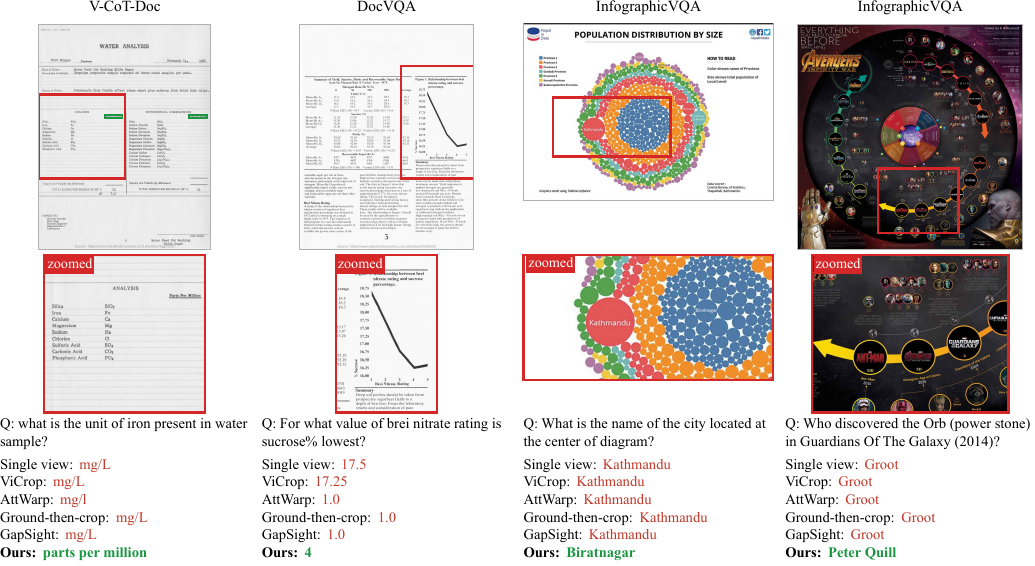}
  \caption{\textbf{Further successes} (Qwen3-VL-4B, 196 tokens). Top: input with the window selected by \ours (red) and, for V-CoT-Doc, the evidence boxes (green). Middle: the zoomed view as the model reads it. Bottom: answers (green: correct, red: wrong).}
  \label{fig:qual_success}
\end{table*}
Fig.~\ref{fig:qual} shows examples of the window selected by \ours and the answers of all methods, and Fig.~\ref{fig:qual_success} further successes.
Fig.~\ref{fig:qual_failure} groups typical failures.
The window can miss the evidence when the budgeted view offers a plausible distractor, such as the same compound in the analysis table instead of the hypothetical-combination table.
A support window can clip a word at its border, and a correct window can still be read on the wrong line of a dense form.
On very tall infographics, a window that contains the evidence is too large to be magnified.

\section{Limitations}
\label{sec:supp_limits}
The \look token is trained once per model, from the model's own answers and without box annotations, and its layer band is chosen per model (Sec.~\ref{sec:supp_exp}).
\ours zooms on every question, and its gains concentrate where the information gap is large: at 392 tokens it leads the strongest of attention-guided zooming, grounding-based zooming and learned crop routing by $3.1$ to $5.3$ points and Q-Zoom by $0.4$, against $12.0$--$14.9$ and $6.3$ points at 98 tokens.
Q-Zoom declines to zoom on $42\%$ of the V-CoT-Doc questions at 392 tokens, a decision that \ours does not make.
On out-of-domain natural images the window is best paired with the stage-2 reader, and naming the object remains stronger on HR-Bench for the 4B model (Tab.~\ref{tab:ood}).
We study text-rich images, where tight budgets make the evidence illegible; zooming into several regions at once and extending the token to video frames are left for future work.

\begin{table*}[!t]\captionsetup{type=figure}
  \centering
  \includegraphics[width=0.90\textwidth]{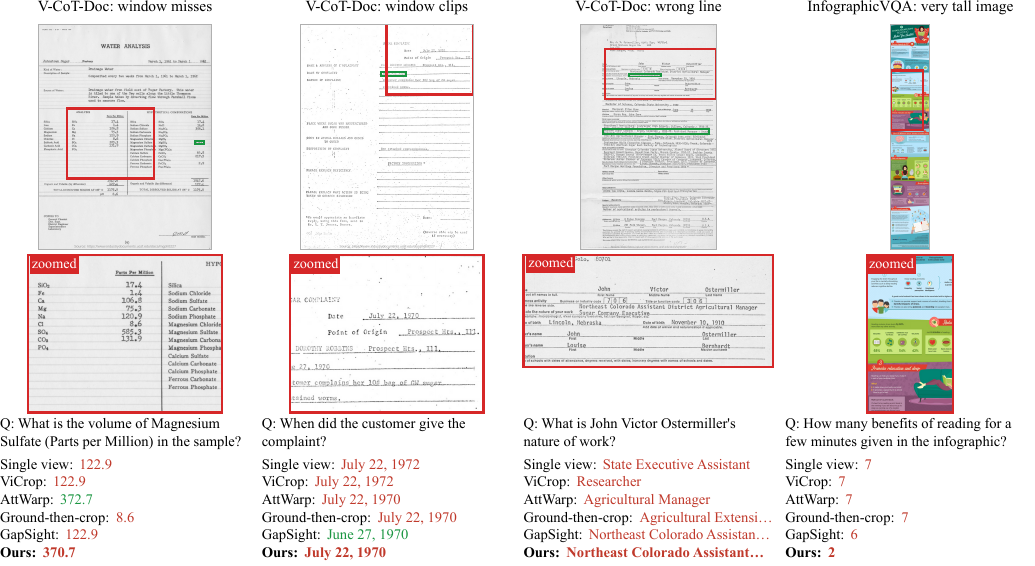}
  \caption{\textbf{Failure cases} (Qwen3-VL-4B, 196 tokens), with the same layout as Fig.~\ref{fig:qual_success}.}
  \label{fig:qual_failure}
\end{table*}
\end{document}